\documentclass{article} 
\usepackage[preprint]{iclr2026_conference}
\usepackage{times}

\usepackage{amsmath,amsfonts,bm}

\def\eqref#1{equation~\ref{#1}}

\def\1{\bm{1}}

\DeclareMathAlphabet{\mathsfit}{\encodingdefault}{\sfdefault}{m}{sl}
\SetMathAlphabet{\mathsfit}{bold}{\encodingdefault}{\sfdefault}{bx}{n}

\usepackage[utf8]{inputenc}
\usepackage[T1]{fontenc}
\usepackage{amsmath, amssymb, amsfonts, amsthm, mathtools}
\usepackage{algorithm}
\usepackage{algpseudocode}
\usepackage{url}
\usepackage{booktabs}
\usepackage{microtype}
\usepackage{xcolor}
\setcitestyle{numbers,square}
\usepackage{hyperref}
\usepackage{graphicx}
\usepackage{caption}   
\usepackage{float}
\usepackage{pgfplots}
\pgfplotsset{compat=1.18}
\usetikzlibrary{calc}
\usepackage{comment}
\usepackage{wrapfig}
\usepackage{colortbl}
\usepackage{array}

\theoremstyle{plain}
\newtheorem{theorem}{Theorem}

\newtheorem{proposition}{Proposition}
\theoremstyle{definition}

\newtheorem{assumption}{Assumption}

\newcommand{\secv}[1]{\textcolor{sec}{#1}}
\newcommand{\pair}[2]{#1\;{\scriptsize\secv{(#2)}}}
\newcommand{\pairT}[2]{#1{\tiny\secv{(#2)}}}

\newcommand{\rponefig}{%
\begin{tikzpicture}[
  frozen/.style={draw=black!45, fill=black!5, rounded corners=2pt,
                 minimum height=5.5mm, minimum width=24mm, align=center,
                 inner sep=1.5pt, font=\footnotesize},
  learn/.style={draw=blue!55, fill=blue!12, line width=.7pt, rounded corners=2pt,
                minimum height=5.5mm, minimum width=13mm, align=center,
                inner sep=1.5pt, font=\footnotesize},
  sm/.style={draw=black!75, circle, line width=.6pt, inner sep=0pt,
             minimum size=3.8mm, font=\footnotesize},
  dot/.style={circle, fill=black!75, inner sep=0pt, minimum size=1.3mm},
  lbl/.style={font=\footnotesize, inner sep=1.5pt},
  flow/.style={line width=.5pt},
  >=latex
]
\foreach \i in {0,1}{
  \begin{scope}[yshift=-\i*29mm]
    \node[dot]    (j\i)  at (0,0)       {};
    \node[frozen] (ev\i) at (0,-6.5mm)  {roll out \& evaluate};
    \node[learn]  (f\i)  at (0,-13.5mm) {$f_\theta$};
    \node[sm]     (s\i)  at (0,-20mm)   {$+$};
    \coordinate   (r\i)  at (15mm,-20mm);
    \draw[->, flow] (j\i)  -- (ev\i);
    \draw[->, flow] (ev\i) -- (f\i);
    \draw[->, flow] (f\i)  -- (s\i);
    \draw[flow]     (j\i)  -- (15mm,0) -- (r\i);
    \draw[->, flow] (r\i)  -- (s\i);
  \end{scope}}
\draw[flow] (0,5.5mm) node[lbl, above] {$\mathbf a_0=\mathbf 0$} -- (j0);
\draw[flow] (s0) -- node[lbl, left=1.5mm] {$\mathbf a_1$} (j1);
\draw[flow] (s1) -- (0,-53mm);
\node[font=\footnotesize] at (0,-54mm) {$\vdots$};
\draw[->, flow] (0,-56.5mm) -- (0,-60mm) node[lbl, below] {$\mathbf a_K$};
\end{tikzpicture}}

\tikzset{
  swatch/.style={draw=black!45, line width=.5pt, rounded corners=0.8pt},
  rp1f/.style ={swatch, draw=blue!60, line width=.8pt, fill=blue!18},
  l2of/.style ={swatch, fill=black!35},
  dmpof/.style={swatch, fill=black!26},
  cemf/.style ={swatch, fill=black!18},
  mppif/.style={swatch, fill=black!11},
  adamf/.style={swatch, fill=black!4},
  rp1/.style ={rp1f,  bar shift=0pt},
  l2o/.style ={l2of,  bar shift=0pt},
  dmpo/.style={dmpof, bar shift=0pt},
  cem/.style ={cemf,  bar shift=0pt},
  mppi/.style={mppif, bar shift=0pt},
  adam/.style={adamf, bar shift=0pt},
}

\definecolor{sec}{gray}{0.42}

\title{Reinforced Planning with Latent World\\ Models}

\author{
Armin Sommer\thanks{Main author, correspondence to \texttt{armin@pantheon.inc}.} \\
Pantheon Industries,
ETH Zürich
\And
Jannik Schilling \\
Pantheon Industries \\
\phantom{ETH Zürich}
}

\begin{document}

\maketitle

\begin{abstract}
Humans solve complex problems by constructing plans and mentally simulating their outcomes with an internal model of the world. Machine learning has made substantial progress in learning world models that predict the consequences of action sequences, yet the procedures used to plan with these models remain largely hand-designed. Most planners rely on fixed search or optimization rules; approaches that learn aspects of search typically imitate a predefined optimizer or use planning to inform an amortized policy, rather improving multi-step plans. We introduce \textbf{Reinforced Planning}, a method that learns the plan-update itself by reinforcing update rules that produce better plans, using gradients propagated through a differentiable world model. We instantiate Reinforced Planning in RP1, which learns a critic over imagined outcomes via temporal-difference learning and a neural plan-improvement operator trained via imagined rollouts with a pretrained world model. RP1 can be trained fully offline without environment interaction; environment episodes are used only for checkpoint selection. Across visual navigation, arm reaching, and robotic manipulation on two world-model backbones, RP1 matches or exceeds existing planners, achieving near-perfect success in several settings while using $1{,}000\times$ fewer world-model rollouts than the strongest alternative (CEM) and planning up to $67\times$ faster under concurrent planners inference.
\end{abstract}

\section{Introduction}
\label{sec:intro}

Humans are commonly understood to solve complex problems by imagining possible
futures and evaluating their consequences. Hippocampal activity can represent
prospective trajectories before an action is taken
\citep{johnson2007neural,pfeiffer2013hippocampal}, supporting the view that the
brain uses a learned cognitive map for internal simulation
\citep{tolman1948cognitive}. Computationally, this separates planning into two
components: a \emph{world model} predicts the consequences of hypothetical
actions, while a \emph{planner} determines how candidate action sequences are
generated, evaluated, and improved.

Machine learning has made substantial progress on the first component. Latent
world models now support high-dimensional visual prediction
\citep{ha2018world,hafner2019planet,hafner2023dreamerv3}, self-supervised
predictive representations \citep{lecun2022path,assran2023ijepa}, and planning
with pretrained, reward-free models
\citep{zhou2024dinowm,sobal2025pldm,maes2026lewm}. Yet the planner operating on
top of these models is still usually hand-designed. Given a candidate action
sequence, the world model can predict its outcome, but it does not specify how
that sequence should be changed to produce a better one. Existing systems
therefore rely on fixed search rules, often requiring thousands of world-model
evaluations per decision and configurations that must be chosen separately for
different tasks and models
\citep{hafner2019planet,hansen2024tdmpc2,zhou2024dinowm,
sobal2025pldm,sv2023gradientplanning}. 

Recent work shows that planning can emerge within neural networks trained through reinforcement learning
\citep{bush2025interpreting,taufeeque2026path,sommer2026planningemergentbehaviorreinforcement},
and learned optimizers show that update rules can be learned from gradients by unrolling the optimizer and training on the resulting losses \citep{andrychowicz2016learning,li2017learning}. It remains open how to learn an explicit planner on top of a separately pretrained world model: existing model-based approaches retain hand-designed update rules, imitate a conventional optimizer, or optimize only the next action rather than an entire action sequence
(Sec.~\ref{sec:background}).

We introduce the \textbf{Reinforced Planning} method and its first implementation, RP1, which learns both how imagined outcomes should be evaluated and how multi-step plans should be improved. RP1 learns a goal-conditioned quasimetric critic
\citep{liu2023metric,wang2023quasimetric} from offline trajectories using temporal-difference learning, then trains a neural planner to repeatedly improve action plans at inference time. At each refinement step, RP1 rolls out the current plan through the world model, evaluates its predicted outcome, and supplies the resulting value and action gradient to a learned update operator. The gradient provides information about the local shape of the objective but does not prescribe the update: the planner learns how to turn that information into a new multi-step plan.

We evaluate RP1 with two pretrained world-model backbones, LeWorldModel and
PLDM, across visual navigation (TwoRoom), continuous-control reaching
(Reacher), and contact-rich manipulation (OGBench Cube). Across these three domains, RP1 matches or exceeds the strongest
existing planners while using only $9$ world-model rollouts per decision,
compared with $9{,}000$ for the strongest competing method, and reduces
planning latency by up to $67\times$ when multiple control loops share one
GPU.

\section{Background}
\label{sec:background}

\paragraph{Fixed or inherited planning rules.}
Most model-based planners use a hand-designed update rule: CEM in PlaNet and
DINO-WM, MPPI in the TD-MPC family, or gradient descent through differentiable
world-model rollouts
\citep{hafner2019planet,hansen2024tdmpc2,zhou2024dinowm,
sobal2025pldm,sv2023gradientplanning}. Universal Planning Networks
optimize multi-step action sequences but fix gradient descent as the optimizer \citep{srinivas2018universal}. DMPO retains an MPPI update and shift operation and learns modifications to them from online task return \citep{sacks2024deep}. L2O-MPC learns the runtime update, but only by
imitating a higher-budget MPPI expert that must be run during training
\citep{sacks2022learning}. In all these, the planning rules are either (partly)
hand-designed or learned from a hand-designed optimizer. Learned optimizers replace hand-designed rules with networks that map gradients to updates and are trained by unrolling \citep{andrychowicz2016learning,li2017learning}; RP1 applies this principle to action plans evaluated through a world model and a learned cost-to-go.

\paragraph{Amortized model-based control.}
The Dreamer methods use their world model to train an amortized policy and do not plan through the model at inference time \citep{dreamerv1}. Diffuser learns a generative model over state--action trajectories and refines them by denoising, rather than rolling out and evaluating candidate plans through a separate world model \citep{janner2022planningdiffusionflexiblebehavior}. The Imagination Based Planner (IBP) and Thinker learn which imagined trajectories to construct or inspect, but the gathered information conditions an amortized policy rather than being used to improve a candidate plan
\citep{pascanu2017learning,chung2024thinker}.

\paragraph{Iterative next-action optimization.}
Iterative Amortized Policy Optimization (IAPO) learns an iterative optimizer for the current-state action
distribution $\pi(a_t\mid s_t)$ \citep{marino2021iterative}. Even in its
model-based variant, future actions in a rollout remain policy
outputs rather than jointly optimized variables, so the learned optimization improves one action rather than a multi-step plan.

\paragraph{Objectives for imagined plans.}
JEPA-based world models typically score imagined outcomes
by their Euclidean distance to the goal in latent space \citep{maes2026lewm,sobal2025pldm,zhou2024dinowm}. However, latent proximity need not reflect
temporal reachability (Theorem~\ref{thm:latent-geometry}), and learned reachability objectives can outperform latent
distance
\citep{li2026euclideanproximityrepairinglatent}. Scoring outcomes by value estimates is well established in model-based control
\citep{hansen2022tdmpc,hansen2024tdmpc2}. We therefore estimate temporal
cost-to-go from experience using a goal-conditioned quasimetric-style critic
\citep{kaelbling1993learning,schaul2015universal,andrychowicz2017hindsight,
hartikainen2020dynamical,liu2023metric,wang2023quasimetric,wang2023optimal}.

\paragraph{Scope of comparison.}
We compare against planners: methods that roll out candidate plans through the world model at inference time and revise them based on the predicted outcome, including the learned planners DMPO and L2O-MPC. Amortized policies are a different family and outside our scope. A planner checks each decision against the world model before executing it, and RP1 uses the same interface as the planners it replaces, so it can be substituted into existing latent-planning pipelines without changing the world model.

\section{Preliminaries}
\label{sec:prelim}

We consider a goal-conditioned MDP
$(\mathcal S, \mathcal A, \mathcal T, g, \rho_0)$ with
state space $\mathcal S$, action space
$\mathcal A \subset \mathbb R^{|a|}$, goal state $g\in \mathcal
S$, transition function $\mathcal T: \mathcal S \times \mathcal A \to \mathcal S$ and
initial-state distribution $\rho_0$. The agent does not observe $s$ directly, but instead receives an observation $o$ in the form of a visual image.

\paragraph{World model.}
A world model predicts the next state, given the current state and some candidate action. Specifically, the neural network maps an observation $o_t$ to its latent representation $z_t$ through an encoder
$z_t = E_\phi(o_{t}) \in
\mathcal Z$. The world model then predicts the next latent state via a prediction map $h_\phi (z, a) = \hat z$, where $\hat z$ is the predicted (or \textit{imagined})
latent of the next state. We define the $N$-step rollout operator of the world model via
\begin{align}
  H_\phi(\mathbf a, \hat z_0)
  &:= h_\phi\bigl( h_\phi( \cdots h_\phi(\hat z_0, a_0) \cdots, a_{N-2}),\,
      a_{N-1} \bigr)
  \;=\; \hat z_N,
  \label{eq:rollout}
\end{align}
a composition of $N$ forward rolls of the world model.

\paragraph{Model-based Planning.}
Given a start latent $z_t$ and encoded goal $z_g$, a plan is scored by a
terminal cost $C$ applied to the final latent state, as predicted by the world
model $H_\phi$. This plan aims to optimize the objective
\begin{equation}
  J(\mathbf a; z_t, z_g)
  \;:=\; C\bigl( H_\phi(\mathbf a, z_t),\, z_g \bigr).
  \label{eq:plan-objective}
\end{equation}
A model-based planner maintains a planner state $\omega_k$, which may contain
a candidate plan, a population of plans, a sampling distribution, or optimizer
state. Given the fixed world model and planning objective, it repeatedly
applies an update rule
\[
F:\omega_k \mapsto \omega_{k+1},
\]
using world-model evaluations to update that state. The final planner state
determines the returned action sequence. RP1 is the special case in which
$\omega_k$ is simply the current action plan $\bf a_k$.

\section{Reinforced Planning}
Our method follows an actor-critic architecture: a \emph{critic} scores latent
states with respect to a goal, and an actor, here a \emph{learned
planner}, optimizes the action plan against the critic's final state estimate.

\paragraph{Critic.}
The critic is a goal-conditioned value function $V_\psi({z}_t, {z}_g)$ that estimates the cost-to-go from latent state ${z_t}$ to an encoded goal state ${z}_g$. Here, lower values correspond to fewer steps to goal and thus signify occupancy of better states. We learn this critic via offline temporal-difference (TD) learning, although it could in theory be any cost-function. During planning, the critic only ever evaluates
terminal states $\hat{z}_N$ produced by the rollout operator.

\paragraph{Planner.}
The planner is a learned operator
\begin{equation}
\mathcal F_\theta:
\left(
\underbrace{\mathbf a_k}_{\text{current plan}},
\underbrace{v_k}_{\text{terminal value}},
\underbrace{\mathbf g_k}_{\text{value gradient}}
\right)
\longmapsto
\underbrace{\mathbf a_{k+1}}_{\text{improved plan}}.
\end{equation}
parametrized by $\theta$, that outputs an improved plan from the current plan, the critic's value at
the plan's terminal state, and the plan's value gradient. Starting from an initial plan $\mathbf{a}_0$, planning does three things per step:
\begin{align}
  \text{\emph{roll out:}} \quad
    & \hat{z}_N^{(k)} = H_\phi\bigl(\mathbf{a}_k, {z}_t\bigr),
    \label{eq:plan-rollout}\\[2pt]
  \text{\emph{evaluate:}} \quad
    & v_k = V_\psi\bigl(\hat{z}_N^{(k)}, {z}_g\bigr),
    \qquad
    \mathbf{g}_k = \nabla_{\mathbf{a}_k} V_\psi\bigl(\hat{z}_N^{(k)}, {z}_g\bigr),
    \label{eq:plan-evaluate}\\[2pt]
  \text{\emph{improve:}} \quad
    & \mathbf{a}_{k+1} =  \mathcal{F}_\theta\bigl(\mathbf{a}_k,\, v_k,\, \mathbf{g}_k\bigr).
    \label{eq:plan-refine}
\end{align}
The optimized plan is the final iteration, $\mathbf{a}^\star = \mathbf{a}_K$.
While the planner has access to the value and gradient, it is not constrained to follow the plan's gradient \(-\mathbf g_k\) and can learn when to trust and distrust it. Task-specific information reaches the planner only through \(v_k\) and \(\mathbf g_k\), forcing it to learn a plan-update rule rather than a direct state-and-goal-to-action mapping.

\paragraph{Reinforcing good planning rules.}
Applying \(\mathcal F_\theta\) produces an imagined optimization trajectory
\begin{align}
\mathbf a_0
\xrightarrow{\mathcal F_\theta}
\mathbf a_1
\xrightarrow{\mathcal F_\theta}
\cdots
\xrightarrow{\mathcal F_\theta}
\mathbf a_K,
\end{align}
where each new plan is obtained by applying the same learned update rule to the preceding plan. The planner is trained to minimize the terminal cost-to-go predicted by the frozen world model \(H_\phi\) and value function $V_\psi$:
\begin{align}
\theta^\star
=
\arg\min_\theta
\mathbb E_{(z_0,z_g)\sim\mathcal D}
\left[
V_\psi\!\left(H_\phi(\mathbf a_K,z_t),z_g\right)
\right] + \mathcal C,
\end{align}
where $\mathcal C$ is some regularizer on intermediate plans' value. Update rules that lower the predicted cost-to-go are thereby reinforced in the shared parameters of $\mathcal F_\theta$, while updates that raise it are suppressed. The planner therefore learns rules to improve action sequences, not the action sequences themselves.


\section{Implementation}

\noindent
\begin{minipage}[t]{0.72\textwidth}
\vspace{0pt}
As a first realization of a Reinforced Planner, we implement a residual version we call \emph{RP1}. Starting from $\mathbf{a}_0 = \mathbf {0}$, the planner optimizes the action trajectory via
\begin{equation}
\begin{aligned}
  \mathcal F_\theta&\bigl(\mathbf{a}_k,\, v_k,\, \mathbf{g}_k\bigr) \\[2pt]
  &= \operatorname{clip}_{[-a_{\max},a_{\max}]}\!\Bigl(
      \underbrace{\mathbf{a}_k}_{\text{prev.\ plan}}+
      \underbrace{f_\theta\bigl(\mathbf{a}_k,\, v_k,\, \mathbf{g}_k\bigr)}_{\text{residual } \Delta\mathbf{a}_k}
    \Bigr),
\end{aligned}
  \label{eq:residual-update}
\end{equation}
where $f_\theta \colon \mathcal{A}^{N} \times \mathbb{R} \times \mathbb{R}^{N \times |a|} \to \mathbb{R}^{N \times |a|}$
is a neural network producing the plan change
$\Delta\mathbf{a}_k = f_\theta(\mathbf{a}_k, v_k, \mathbf{g}_k)$.
\par
\medskip
\noindent
This residual update $\Delta\mathbf{a}_k$ onto the previous plan $\mathbf{a}_k$ keeps the gradient flow stable to avoid the vanishing gradient problem\cite{Hochreiter1991,resnet}. The clip is a projection of each plan iterate onto the box $[-a_{\max},a_{\max}]$ (an $\ell_\infty$ constraint on the action trajectory, not on the update), bounding actions to $a_{\max}$ standard deviations of the demonstrated distribution so that rollouts stay on the world model's support. We use open-loop planning for our experiments.
\par
\medskip
\noindent
We realize the critic as a metric residual network~\cite{liu2023metric} trained with offline value learning through Hindsight Experience Replay\cite{kostrikov2022offline,andrychowicz2017hindsight}. For specifics see Appendix~\ref{app:method}.
\end{minipage}\hfill
\begin{minipage}[t]{0.24\textwidth}
\vspace{0pt}
\centering
\captionsetup{type=figure,font=small,skip=4pt}
\rponefig
\caption{\textbf{RP1 visual.}}
\label{fig:rp1-residual}
\end{minipage}

\section{Theoretical Results}
\label{sec:optimality}

We formalize two motivations for Reinforced Planning. First, predictive
world-model learning alone does not identify a Euclidean latent geometry that
is necessarily suitable for planning. Second, we give an expressivity result
for learned planning: under a matched information interface, an expressive
learned update rule can perform task-dependent search behavior that no
single fixed optimizer configuration can realize across a heterogeneous task
distribution. The second result establishes merely a
representational advantage of learned search rules rather than a performance
guarantee for RP1.

\subsection{Latent Norms and Cost-to-go}

Recall the encoder \(E_\phi\) and latent transition model \(h_\phi\) from
Section~\ref{sec:prelim}. For notational simplicity, we write \(E_\phi(s)\)
for the encoding of the observation generated by state \(s\). Let
\(\mathcal T:\mathcal S\times\mathcal A\to\mathcal S\) denote deterministic
environment dynamics. We call the latent world model exact when
\begin{equation}
h_\phi(E_\phi(s),a)
=
E_\phi(\mathcal T(s,a))
\qquad
\text{for all }(s,a)\in\mathcal S\times\mathcal A.
\label{eq:exact-world-model}
\end{equation}

\begin{theorem}[Prediction does not identify Euclidean latent geometry]
\label{thm:latent-geometry}
Suppose \((E_\phi,h_\phi)\) is exact. If the latent displacements from some
state \(s\) to two goals \(g_1\) and \(g_2\) are linearly independent, then
there exist two equally exact latent reparameterizations that reverse which
goal is closer to \(s\) under Euclidean distance. The ratio between the two
distances can be made arbitrarily large.
\end{theorem}

\begin{proof}[Proof sketch]
Any invertible linear change of latent coordinates can be absorbed into both
the encoder and transition model without changing predictive exactness. By
mapping the two goal displacements to separate coordinate axes and stretching
either axis, either goal can be made arbitrarily farther than the other. The
full proof is given in Appendix~\ref{app:latent-geometry}.
\end{proof}

Theorem~\ref{thm:latent-geometry} does not imply that latent distance is
necessarily a poor planning objective. Rather, it shows that predictive
accuracy alone cannot determine whether it is a good one: two equally
predictive world models can rank the same candidate goals in opposite orders.
Agreement between latent distance and temporal cost-to-go is therefore an
additional property that must be learned or imposed separately. We learn this
property by training a goal-conditioned critic.

This non-identifiability holds even when the environment is reversible and
temporal reachability is symmetric. Appendix~\ref{app:theory-quasimetric}
gives the complementary result that temporal reachability can additionally be
asymmetric, in which case no symmetric latent norm can represent it exactly.

\subsection{Advantage of Learned Planning under Task Heterogeneity}

For a fixed world model, action space, planning horizon, and objective, let
\(x=(z_t,z_g)\sim\mu\) denote a planning task. A planner state \(\omega_k\in\Omega_{\mathrm{pl}}\subseteq\mathbb R^{d_\omega}\) contains all information carried from one refinement round to the next, including the current candidate plans and any optimizer memory; we assume \(\Omega_{\mathrm{pl}}\) is nonempty and compact. At each round, task-dependent information is exposed through a fixed interface \(\mathcal I\), shared by the learned and fixed update rules compared below. Starting from a shared initialization
\(\omega_0\), an update rule \(F\) is applied for \(K\) rounds:
\begin{equation}
\omega_{F,0}(x)
=
\omega_0,
\qquad
\omega_{F,k+1}(x)
=
F\!\left(
\omega_{F,k}(x),
\mathcal I\bigl(x,\omega_{F,k}(x)\bigr)
\right).
\label{eq:abstract-planner-update}
\end{equation}
Its expected loss is
\begin{equation}
\mathcal L(F)
=
\mathbb E_{x\sim\mu}
\left[
\ell\bigl(x,\omega_{F,K}(x)\bigr)
\right],
\label{eq:planner-loss}
\end{equation}
where \(\ell\) is the cost of the plan returned from the final planner state.

Let \(\mathcal W\) denote the compact set of feasible planner inputs, and let
\(\mathfrak F_{\mathcal I}\) denote the continuous feasible update rules
\(F:\mathcal W\to\Omega_{\mathrm{pl}}\). The precise ambient spaces and
regularity conditions are given in Appendix~\ref{app:heterogeneity}.

\begin{assumption}[Universal search-rule approximation]
\label{as:planner-universality}
Assume an abstract learned-planner class \(\{\mathcal F_\theta:\theta\in\Theta\}\subseteq\mathfrak F_{\mathcal I}\), whose members receive the full feasible input \(w\in\mathcal W\), can uniformly approximate every rule in \(\mathfrak F_{\mathcal I}\):
for every \(F\in\mathfrak F_{\mathcal I}\) and every \(\varepsilon>0\),
there exists \(\theta\in\Theta\) such that
\begin{equation}
\sup_{w\in\mathcal W}
\left\|
\mathcal F_\theta(w)-F(w)
\right\|
<
\varepsilon.
\end{equation}
\end{assumption}

Now let \(\mathcal X_1,\ldots,\mathcal X_r\) be a measurable partition of the
task distribution, with
$
p_i
=
\Pr(x\in\mathcal X_i)
>
0.$
For any update rule \(F\), define its regional loss by
\begin{equation}
\mathcal L_i(F)
=
\mathbb E
\left[
\ell\bigl(x,\omega_{F,K}(x)\bigr)
\,\middle|\,
x\in\mathcal X_i
\right].
\end{equation}

Let \(\mathcal B\) be a family of fixed search configurations. Each
\(B\in\mathcal B\) induces an update rule \(F_B\), and the same configuration
is used on every task. We write
\begin{equation}
\mathcal L_i(B)
=
\mathcal L_i(F_B),
\qquad
\mathcal L(B)
=
\sum_{i=1}^r p_i\mathcal L_i(B).
\end{equation}

\begin{theorem}[Strict advantage under task heterogeneity]
\label{thm:fixed-config-superiority}
Under Assumption~\ref{as:planner-universality} and the regularity, regional
incompatibility, and interface-composability conditions stated in
Appendix~\ref{app:heterogeneity},
\begin{equation}
\inf_{\theta\in\Theta}
\mathcal L(\mathcal F_\theta)
\leq
\sum_{i=1}^r
p_i
\min_{B\in\mathcal B}
\mathcal L_i(B)
<
\min_{B\in\mathcal B}
\mathcal L(B).
\end{equation}
Thus, a sufficiently expressive learned planner can strictly outperform every single fixed search configuration by adapting its update behavior across task regions through the shared interface.
\end{theorem}

Note that theorem~\ref{thm:fixed-config-superiority} is an expressivity result, not a
claim that learning an optimizer necessarily yields better empirical
performance. Applying the result specifically to RP1 would require additional regularity conditions ensuring that
$\mathbf g=\nabla_{\mathbf a}V_\psi(H_\phi(\mathbf a,z_t),z_g)$ is well-defined
and continuous on the relevant domain.

\section{Experiments}
\label{sec:experiments}
\vspace{-0.6em}

\begin{figure}[H]
  \centering
  \includegraphics[
    width=0.65\linewidth,
    trim=0 0 0 10,
    clip
  ]{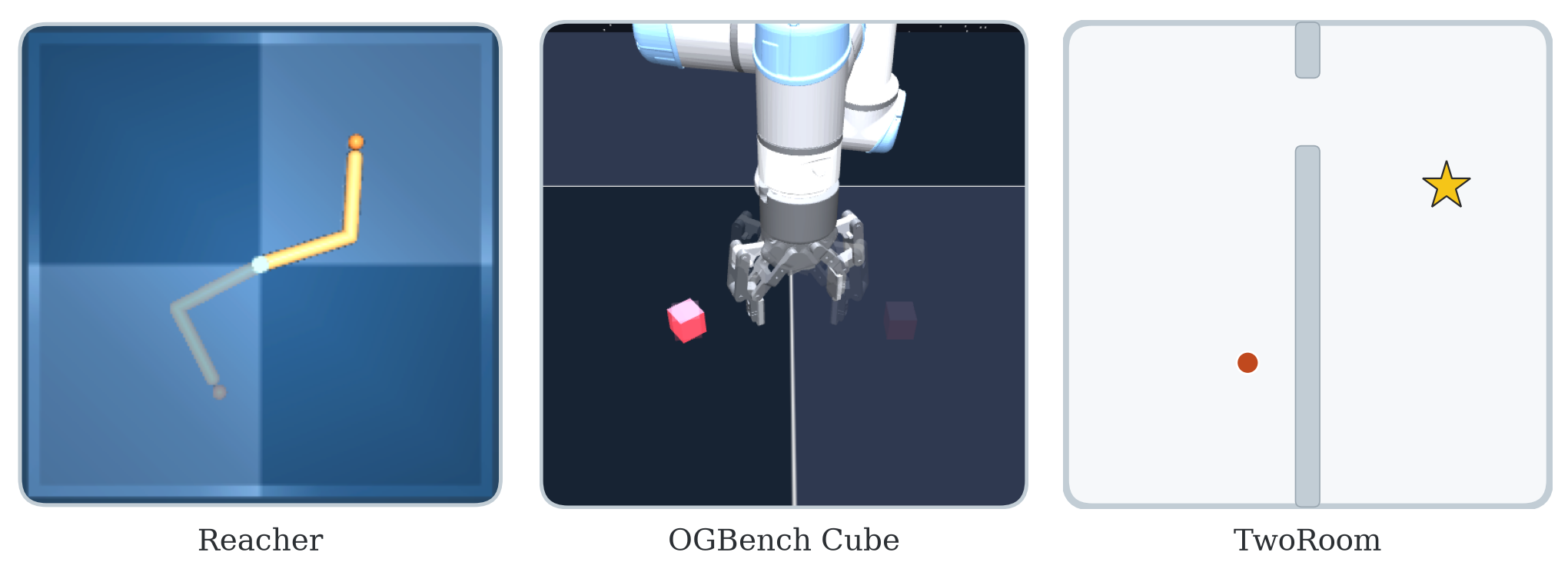}
  \vspace{-0.6em}
  \caption{\textbf{Experiment environments.} In Reacher (left) the agent moves a two-link arm to a goal configuration, here shaded. In OGBench Cube (middle) a robot arm picks up a cube and moves it to a goal position, also shaded. In TwoRoom (right) an agent navigates to a goal position, (here the star).}
  \label{fig:environments}
\end{figure}

\paragraph{Evaluation design.} We evaluate RP1 in three visual-control domains, TwoRoom, Reacher, and OGBench Cube, on two world-model backbones: LeWorldModel (LeWM) and PLDM. Each backbone is trained separately for each environment, so every (environment, backbone) pair has its own latent space and dynamics. RP1 is trained per pair for the same reason: its inputs, the value $v_k$ and gradient $\mathbf g_k$, are computed through a specific world model and critic, and their scale and structure change with the latent space. Thus, moving to a new world model base requires retraining the planner's weights but no new configuration search. We use benchmarks from the StableWorldModel environment \citep{maes_lld2026swm}.

\paragraph{Baselines and controlled comparisons.} We compare RP1 against three popular hand-designed planning algorithms and two partly-learned hybrid planners. Essentially all recent latent-planning systems report using one of the hand-designed planners or a close variant \citep{hafner2019planet,hansen2024tdmpc2,zhou2024dinowm,sobal2025pldm,sv2023gradientplanning}. We evaluate each with both latent distance $\lVert \hat z_N-z_g\rVert_2^2$ and the learned objective $V(\hat z_N,z_g)$. The hybrid planners DMPO and L2O-MPC use a learned critic, following their original design \cite{sacks2024deep, sacks2022learning}.

\paragraph{Evaluation metrics.} We report task success and planning compute cost, measured as the number of world-model rollouts per decision. All planners use identical action chunking: each planned action comprises five primitive actions, and each planner optimizes a sequence of five such chunks, corresponding to a horizon of 25 primitive actions. All methods therefore plan over the same horizon in the same normalized action space. Full evaluation details are provided in Appendix~\ref{app:results}.

\subsection{TwoRoom}
\label{sec:tworoom}

TwoRoom tests whether model-based agents are capable of appropriate planning when geometric proximity differs from temporal reachability. The agent must pass through a doorway to reach the opposite room, so states that are geometrically close across the wall might still require a long detour.

\begin{figure}[H]
  \centering

  \makebox[\linewidth][c]{%
    \begin{minipage}[c]{0.04\linewidth}
      \centering\small (a)
    \end{minipage}%
    \hspace{4pt}%
    \begin{minipage}[c]{0.75\linewidth}
      \includegraphics[width=\linewidth]{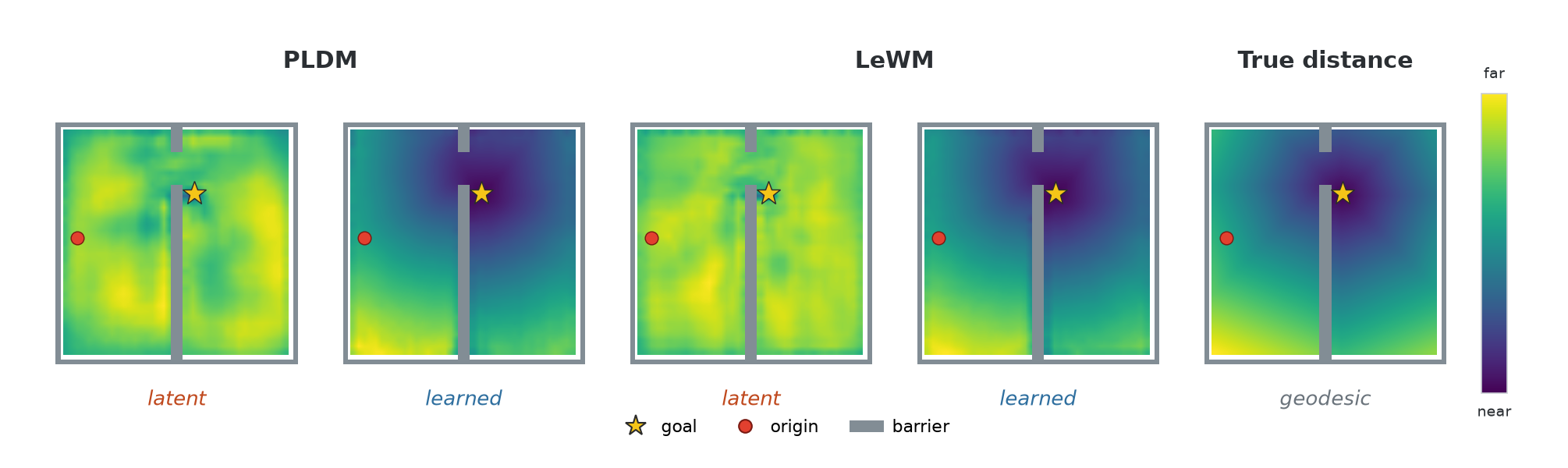}
    \end{minipage}%
  }

  \vspace{6pt}

  \makebox[\linewidth][c]{%
    \begin{minipage}[c]{0.04\linewidth}
      \centering\small (b)
    \end{minipage}%
    \hspace{4pt}%
    \begin{minipage}[c]{0.75\linewidth}
      \includegraphics[width=\linewidth]{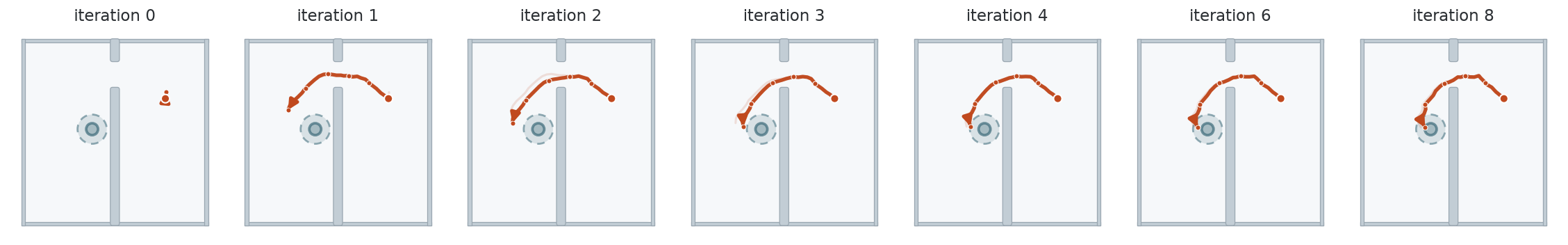}
    \end{minipage}%
  }

  \caption{\textbf{RP1 in TwoRoom.}
  (a) The learned critic better captures temporal cost-to-go than latent
  $L_2$ distance. (b) RP1 iteratively refines its plan, with later updates
  focusing on fine corrections to the final actions.
  Additional visualizations are provided in Appendix~\ref{app:results-tworoom}.}
  \label{fig:tworoom-analysis}
\end{figure}
 We find that the latent-distance objective does not capture distance-to-goal in the queried world-models, whereas a learned value function saturates the benchmark across implemented planners at $25-$step horizon. Once planners are given the learned value critic, their performance largely converges: most methods reach near-saturated success, despite using very different search rules and compute budgets.

\renewcommand{\secv}[1]{\textcolor{black!90}{#1}}

\begin{table}[H]
  \centering
  \captionsetup{font=small,skip=3pt}
  \caption*{\textbf{TwoRoom}}
  \label{tab:tworoom}
  \footnotesize
  \renewcommand{\arraystretch}{1.1}
  \begin{tabular*}{\linewidth}{@{\extracolsep{\fill}}lrcccc@{}}
  \toprule
  & & \multicolumn{2}{c}{LeWM} & \multicolumn{2}{c}{PLDM}\\
  \cmidrule(lr){3-4}\cmidrule(l){5-6}
  planner & rollouts & $25$ steps & $100$ steps & $25$ steps & $100$ steps\\
  \midrule
  \multicolumn{6}{@{}l}{\itshape latent $L_2$ \;{\scriptsize\secv{(value critic)}}}\\
  \addlinespace[1pt]
  CEM  & $9000$ & \pair{84.0}{\textbf{100.0}} & \pair{13.3}{73.0}
                & \pair{93.3}{\textbf{100.0}} & \pair{52.0}{49.3}\\
  MPPI & $9000$ & \pair{70.7}{95.3} & \pair{20.0}{68.0}
                & \pair{64.0}{95.7} & \pair{33.3}{53.3}\\
  Adam & $3000$ & \pair{94.7}{96.3} & \pair{24.0}{68.0}
                & \pair{90.7}{94.0} & \pair{42.0}{53.0}\\
  \addlinespace[2pt]\midrule\addlinespace[1pt]
  DMPO    & $256$ & 73.3 & 69.3 & 76.7 & 44.0\\
  L2O-MPC & $256$ & 98.0 & 92.7 & 99.3 & 46.7\\
  \addlinespace[2pt]\midrule\addlinespace[1pt]
  \textbf{RP1 (ours)} & $9$ & \textbf{100.0} & \textbf{99.3} & 98.0 & \textbf{94.0}\\
  \bottomrule
  \end{tabular*}
  \vspace{3pt}
  \caption{\centering{\textbf{Success rate (\%).} Main number: latent $L_2$; in
  parentheses: the learned value critic RP1 deploys. \textbf{Bold} marks the strongest
  baseline per column.}}
\end{table}

\subsection{Reacher}
\label{sec:highdim}

Reacher is a two-link arm under torque control, observed only as
visual frames. The task is to bring both joints into a target
configuration. Success follows the benchmark's first-hit convention at a
loose and a tight tolerance ($\tau{=}0.1$ and $\tau{=}0.05$\,rad). As in all
domains, planners optimize five blocks of five primitive actions, so the
planning horizon exactly covers the nominal $25$-step distance to the goal.

\renewcommand{\secv}[1]{\textcolor{black!90}{#1}}

\begin{table}[H]
  \centering
  \captionsetup{font=small,skip=3pt}
  \caption*{\textbf{Reacher}}
  \label{tab:reacher}
  \footnotesize
  \renewcommand{\arraystretch}{1.1}
  \begin{tabular*}{\linewidth}{@{\extracolsep{\fill}}lrcccc@{}}
  \toprule
  & & \multicolumn{2}{c}{LeWM} & \multicolumn{2}{c}{PLDM}\\
  \cmidrule(lr){3-4}\cmidrule(l){5-6}
  planner & rollouts & $\tau{=}.1$ & $\tau{=}.05$ & $\tau{=}.1$ & $\tau{=}.05$\\
  \midrule
  \multicolumn{6}{@{}l}{\itshape latent $L_2$ \;{\scriptsize\secv{(value critic)}}}\\
  \addlinespace[1pt]
  CEM  & $9000$ & \pair{98.7}{88.0} & \pair{80.3}{63.3}
                & \pair{96.7}{75.3} & \pair{80.0}{56.3}\\
  MPPI & $9000$ & \pair{90.0}{89.0} & \pair{69.3}{70.3}
                & \pair{89.3}{86.0} & \pair{63.3}{65.3}\\
  Adam & $3000$ & \pair{94.0}{92.7} & \pair{66.0}{71.7}
                & \pair{94.3}{92.7} & \pair{66.0}{68.7}\\
  \addlinespace[2pt]\midrule\addlinespace[1pt]
  DMPO    & $256$ & 40.7 & 20.7 & 36.0 & 19.3\\
  L2O-MPC & $256$ & 68.0 & 46.7 & 67.3 & 39.3\\
  \addlinespace[2pt]\midrule\addlinespace[1pt]
  \textbf{RP1 (ours)} & $9$ & \textbf{100.0} & \textbf{95.3} & \textbf{99.3} & \textbf{87.3}\\
  \bottomrule
  \end{tabular*}
  \vspace{3pt}
  \caption{\centering{\textbf{First-hit success (\%).} Goal tolerance $\tau$ (rad). Main
  number: latent $L_2$; in parentheses: the learned value critic RP1 deploys.
  \textbf{Bold} marks the strongest baseline per column.}}
\end{table}

Even this near-saturated task discriminates between planners once the
tolerance is tightened. At $\tau{=}0.1$, the strongest planners all bring the
arm into the neighborhood of the goal, with success rates differing by only a
few percentage points. RP1 does so using three orders of magnitude fewer
world-model rollouts. All methods degrade with decreased goal tolerance $\tau$, but RP1
retains the highest success rate in every column; relative to the strongest
baseline, its margin increases from approximately one percentage point at
$\tau{=}0.1$ to $15$ points on LeWM and $7$ points on PLDM at
$\tau{=}0.05$.

\subsection{OGBench Cube}
\label{sec:ogbench}

OGBench Cube~\citep{park2025ogbench} is a vision-based manipulation
benchmark: a robot arm must pick up a cube and place it at a goal
position, observed only from pixels, with goals placed $25$ or $100$
primitive steps away ($h25$, $h100$). Alongside the regular OGBench score we report a \emph{hard} score, the same runs normalized
against the measured no-op floor $f$ as $(s-f)/(100-f)\cdot 100$, which
measures the fraction of headroom above doing nothing that a planner
actually converts. The hard score is our primary number; easy is kept
for comparability with the benchmark's convention.

%
\vspace{0.8em}

\begin{table}[H]
  \centering
  \label{tab:ogbench}
  \captionsetup{font=small,skip=3pt}
  \caption*{\textbf{OGBench Cube.}}
  \footnotesize
  \setlength{\tabcolsep}{5pt}
  \renewcommand{\arraystretch}{1.15}
  \begin{tabular}{@{}l@{\hspace{1.8em}}cc@{\hspace{1.5em}}cc@{\hspace{2.6em}}cc@{\hspace{1.5em}}cc@{}}
  \toprule
  & \multicolumn{4}{c}{LeWM} & \multicolumn{4}{c}{PLDM}\\
  \cmidrule(lr){2-5}\cmidrule(l){6-9}
  & \multicolumn{2}{c}{$25$ steps} & \multicolumn{2}{c}{$100$ steps}
    & \multicolumn{2}{c}{$25$ steps} & \multicolumn{2}{c}{$100$ steps}\\
  \cmidrule(lr){2-3}\cmidrule(lr){4-5}\cmidrule(lr){6-7}\cmidrule(l){8-9}
  planner\,{\scriptsize\secv{(roll.)}} & easy & hard & easy & hard & easy & hard & easy & hard\\
  \midrule
  \multicolumn{9}{@{}l}{\itshape latent $L_2$ \,{\scriptsize\secv{(value-critic)}}}\\
  CEM \,{\scriptsize\secv{9000}} & \pairT{74.0}{80.0} & \pairT{40.9}{54.5} & \pairT{58.0}{77.7} & \pairT{23.2}{59.2}
       & \pairT{62.7}{75.3} & \pairT{15.2}{43.9} & \pairT{58.7}{69.3} & \pairT{24.5}{43.8}\\
  MPPI \,{\scriptsize\secv{9000}} & \pairT{56.7}{67.0} & \pairT{1.6}{25.0} & \pairT{46.0}{58.3} & \pairT{1.2}{23.7}
       & \pairT{60.0}{65.3} & \pairT{9.1}{21.1} & \pairT{47.3}{53.0} & \pairT{3.6}{14.0}\\
  Adam \,{\scriptsize\secv{3000}} & \pairT{74.0}{77.3} & \pairT{40.9}{48.4} & \pairT{57.3}{70.7} & \pairT{21.9}{46.4}
       & \pairT{63.3}{66.0} & \pairT{16.6}{22.7} & \pairT{55.3}{57.0} & \pairT{18.2}{21.3}\\
  \midrule
  DMPO \,{\scriptsize\secv{256}} & 66.0 & 22.7 & 54.0 & 15.9
       & 61.3 & 12.0 & 50.0 & 8.5\\
  L2O-MPC \,{\scriptsize\secv{256}} & 64.0 & 18.2 & 54.0 & 15.9 & 56.0 & 0.0 & 45.3 & 0.0\\
  \midrule
  \textbf{RP1 (ours)} \,{\scriptsize\secv{9}} & \textbf{92.7} & \textbf{83.4} & \textbf{85.3} & \textbf{73.1} & \textbf{85.3} & \textbf{66.6} & \textbf{83.7} & \textbf{70.2}\\
  \bottomrule
  \end{tabular}
  \caption{\centering{\textbf{Success rate ($\%$).} Easy is the raw rate; hard is normalized against
  the no-op floor ($56.0$ at $25$ steps, $45.3$ at $100$), clipped at $0$ for methods below
  it. \textbf{Bold} marks the best entry per column.}}
\end{table}

\subsection{World-Model Hallucination and Dyna Finetuning}
\label{sec:dyna-main}
Training the planner through a frozen world model lets it exploit model
error. Inspecting RP1's failure episodes in OGBench Cube, we found the world
model \emph{hallucinating} contact outcomes: for LeWM, grasps that
miss the cube are nevertheless predicted "magically" to attach it to the arm. No improvement in search can fix such hallucinations. We therefore correct the model rather than the planner: one Dyna iteration \citep{sutton1991dyna} deploys the
trained planner, collects its (failure) rollouts, finetunes the world
model on them, and retrains the planner (Appendix~\ref{sec:dyna}).

\begin{table}[H]
  \centering
  \captionsetup{font=small,skip=3pt}
  \footnotesize
  \setlength{\tabcolsep}{12pt}
  \renewcommand{\arraystretch}{1.15}
  \begin{tabular}{@{}lcccc@{}}
  \toprule
  & \multicolumn{2}{c}{LeWM} & \multicolumn{2}{c}{PLDM}\\
  \cmidrule(lr){2-3}\cmidrule(lr){4-5}
  & easy & hard & easy & hard\\
  \midrule
  RP1 \,{\scriptsize\secv{(pretrained)}} & 92.7 & 83.4 & 85.3 & 66.6\\
  RP1\textsubscript{Dyna} \,{\scriptsize\secv{(finetuned)}}
      & \textbf{95.0} & \textbf{88.6} & \textbf{93.3} & \textbf{84.8}\\
  \addlinespace[1pt]
  \bottomrule
  \end{tabular}
  \vspace{3pt}
  \caption{\textbf{Effect of Dyna at $h25$ (\%).} Rollouts are collected at both goal
  horizons by the deployed planners and the locked configuration is retrained on the
  finetuned model; we run two rounds and report the one chosen on validation
  (round~1 for LeWM, round~2 for PLDM)..}
  \label{tab:dyna}
\end{table}

One iteration recovers a large part of the exploitation gap, and we found empirically that the "grasp-and-miss" hallucinations were significantly reduced in LeWM. However, despite mitigating exploitation,
characterizing when it recurs remains open.

\subsection{Planning Speed}
\label{sec:planning-speed}

We measure end-to-end planning latency on OGBench Cube at the 25-step goal offset with LeWM, including the complete computation from the input latents to the returned
action plan, including RP1's backward passes through the world model and critic. All methods run in fp32 on a single NVIDIA H200 and are
benchmarked using both CUDA-graph-captured and eager execution, with the
faster mean reported. We consider one planner running alone ($B=1$) and
$50$ independent planners running concurrently on the same GPU ($B=50$),
representing multiple control loops sharing one accelerator. Hand-designed planners use the learned value function as their objective.

\begin{figure}[H]
\centering
\begin{tikzpicture}[font=\footnotesize]
 
\begin{axis}[
  name=A,
  ymode=log, log origin=infty, clip=false,
  width=6.7cm, height=4.6cm,
  axis lines=left, axis line style={black!45, line width=.5pt},
  tick align=outside, tick style={black!45},
  ymajorgrids, grid style={black!12, line width=.4pt},
  minor y tick num=0, minor tick style={draw=none},
  ybar, bar width=4.4mm,
  point meta=explicit symbolic, nodes near coords,
  every node near coord/.append style={font=\scriptsize, inner sep=2.5pt},
  symbolic x coords={RP1,L2O,DMPO,CEM,MPPI,Adam},
  xtick=\empty, enlarge x limits=0.13,
  ymin=8, ymax=2.2e4, ytick={1e1,1e2,1e3,1e4}, yticklabels={10,100,1k,10k},
]
\addplot[rp1]  coordinates {(RP1,30)    [30]};
\addplot[l2o]  coordinates {(L2O,45.3)  [45]};
\addplot[dmpo] coordinates {(DMPO,83.3) [83]};
\addplot[cem]  coordinates {(CEM,390.6) [391]};
\addplot[mppi] coordinates {(MPPI,420.3)[420]};
\addplot[adam] coordinates {(Adam,959.3)[959]};
\draw[black!75, ->, line width=.5pt] (axis cs:CEM,6.5e3) -- (axis cs:RP1,6.5e3);
\node[anchor=south, xshift=4.7mm] at (axis cs:L2O,7.5e3) {$13\times$ faster};
\end{axis}
 
\begin{axis}[
  name=B, at={(A.south east)}, xshift=12mm, anchor=south west,
  ymode=log, log origin=infty, clip=false,
  width=6.7cm, height=4.6cm,
  axis lines=left, axis line style={black!45, line width=.5pt},
  tick align=outside, tick style={black!45},
  ymajorgrids, grid style={black!12, line width=.4pt},
  minor y tick num=0, minor tick style={draw=none},
  ybar, bar width=4.4mm,
  point meta=explicit symbolic, nodes near coords,
  every node near coord/.append style={font=\scriptsize, inner sep=2.5pt},
  symbolic x coords={RP1,L2O,DMPO,CEM,MPPI,Adam},
  xtick=\empty, enlarge x limits=0.13,
  ymin=20, ymax=6e5, ytick={1e2,1e3,1e4,1e5}, yticklabels={100,1k,10k,100k},
]
\addplot[rp1]  coordinates {(RP1,79)      [79]};
\addplot[l2o]  coordinates {(L2O,188.3)   [188]};
\addplot[dmpo] coordinates {(DMPO,236.3)  [236]};
\addplot[cem]  coordinates {(CEM,5308.2)  [5.3k]};
\addplot[mppi] coordinates {(MPPI,10020.3)[10.0k]};
\addplot[adam] coordinates {(Adam,17129.8)[17.1k]};
\draw[black!75, ->, line width=.5pt] (axis cs:CEM,1.6e5) -- (axis cs:RP1,1.6e5);
\node[anchor=south, xshift=4.7mm] at (axis cs:L2O,1.9e5) {$67\times$ faster};
\end{axis}
 
\node[anchor=north, font=\footnotesize] at ([yshift=-3mm]A.south)
  {(a) latency in ms, $B=1$};
\node[anchor=north, font=\footnotesize] at ([yshift=-3mm]B.south)
  {(b) latency in ms, $B=50$};
 
\foreach \sty/\name [count=\i from 0] in
  {rp1f/RP1, l2of/L2O-MPC, dmpof/DMPO, cemf/CEM, mppif/MPPI, adamf/Adam}{
  \node[\sty, minimum width=3.4mm, minimum height=2.6mm, inner sep=0pt,
        anchor=west] (sw\i) at ([xshift=\i*21mm-3mm, yshift=9mm]A.north west) {};
  \node[anchor=west, font=\footnotesize] at ([xshift=1.4mm]sw\i.east) {\name};
}
\end{tikzpicture}
\caption{\textbf{End-to-end planning latency.} On OGBench Cube with LeWM (one NVIDIA H200, fp32), with RP1 $13\times$ faster than CEM for one planner and $67\times$ faster for $50$ concurrent planners.}
\label{fig:speed}
\end{figure}
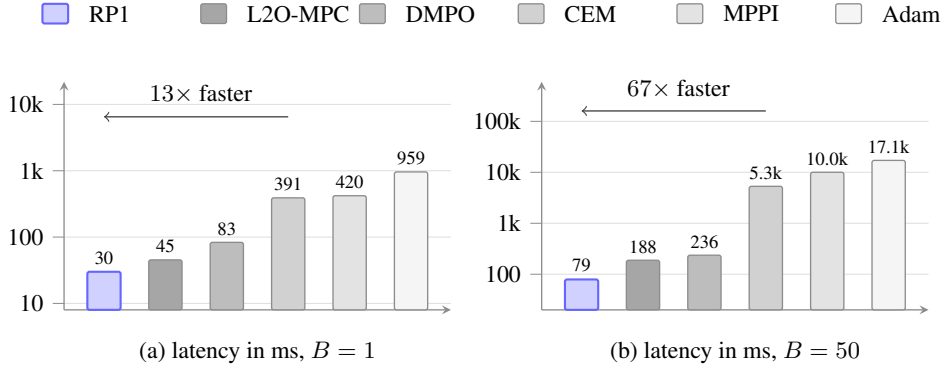

The $1{,}000\times$ reduction in world-model rollouts does not translate
one-for-one into single-planner latency because the GPU can evaluate many of
a sampling planner's candidate trajectories in parallel. Nevertheless, RP1
completes a single planning request $13\times$ faster than the critic-based CEM alternative, the strongest conventional planner by success rate. The
advantage grows substantially under concurrent inference: RP1 processes $50$ concurrently deployed planners 
$67\times$ faster than CEM. RP1 also remains $3\times$ faster than DMPO and
$2.4\times$ faster than L2O-MPC in this setting. Thus, the rollout reduction
becomes most consequential when one accelerator serves several control
loops, such as multiple robot arms planning in tandem.

\section{Discussion}
\label{sec:discussion}

Our results support both motivating hypotheses: where geometric proximity differs from temporal reachability, a learned critic resolves failures that more search does not (Sec.~\ref{sec:tworoom}); and where the objective is already informative, a learned update rule matches or exceeds hand-designed planners with two to three orders of magnitude fewer world-model queries.

\paragraph{Limitations.}
Because RP1 is trained through a frozen world model, it can exploit model errors; Dyna finetuning (Sec.~\ref{sec:dyna-main}) reduces but does not eliminate this. Varying the number of refinement steps $K$ gave improvements up to $K=8$, but for higher $K$ did not noticeably change RP1's success, and for the baselines we use the budgets of prior work \citep{zhou2024dinowm,maes2026lewm,sacks2024deep,sacks2022learning} rather than tracing their success as a function of budget. Unlike hand-designed planners, RP1 must be retrained for each world model and goal horizon. We do not compare against amortized policies trained through the same critic, as we only aim to compare inference-time model-based planners. Checkpoint selection uses simulator evaluation on held-out tasks, a common practice in offline RL \citep{paine2020hyperparameter,kurenkov2022showing}; Appendix~\ref{app:results-general} reports its cost.

\section{AI Use Statement}
Generative Artificial Intelligence aided in the development of this paper. Specifically, it helped with the polishing of prose and mathematics, literature search and code-writing. We have reviewed and verified any LLM-generated code. All ideas behind the method are products of the author's human minds.

\section{Acknowledgements}

\bibliographystyle{plainnat}
\bibliography{references}
\newpage

\appendix
\section{Codebase}
The code is available at \url{https://anonymous.4open.science/r/RLP_original-0A0D}. 

\section{Proofs of Theoretical Results}
\label{sec:optimality-proofs}

We formalize two motivations for Reinforced Planning. First, predictive
world-model learning does not determine a Euclidean latent geometry suitable
for planning. Second, an abstract learned neural planner can adapt its optimization rule to the task, whereas conventional optimizers use one configuration across the task distribution.

\subsection{Proof of Theorem~\ref{thm:latent-geometry}}
\label{app:latent-geometry}

\begin{proof}
Let $A\in\mathbb R^{d\times d}$ be invertible and define
\begin{equation}
E_A(s)
=
A E_\phi(s),
\qquad
h_A(z,a)
=
A h_\phi(A^{-1}z,a).
\label{eq:transformed-world-model}
\end{equation}
Then
\begin{align}
h_A(E_A(s),a)
&=
A h_\phi(A^{-1}A E_\phi(s),a)
\\
&=
A h_\phi(E_\phi(s),a)
\\
&=
A E_\phi(\mathcal T(s,a))
\\
&=
E_A(\mathcal T(s,a)).
\end{align}
Thus, $(E_A,h_A)$ is exact whenever $(E_\phi,h_\phi)$ is exact. The same
argument applied recursively shows that all multi-step trajectories remain
exactly predicted under the transformed coordinates.

Now define
\begin{equation}
u
=
E_\phi(g_1)-E_\phi(s),
\qquad
v
=
E_\phi(g_2)-E_\phi(s).
\end{equation}
Because $u$ and $v$ are linearly independent, there exists an invertible
matrix $B$ such that
\begin{equation}
Bu=e_1,
\qquad
Bv=e_2,
\end{equation}
where $e_1$ and $e_2$ are the first two standard basis vectors.

For any $R>1$, define
\begin{align}
A_1
&=
\operatorname{diag}(R,1,\ldots,1)B,
\\
A_2
&=
\operatorname{diag}(1,R,1,\ldots,1)B.
\end{align}
Both matrices are invertible and therefore induce exact latent world models.
Under the first transformation,
\begin{equation}
\frac{
\lVert E_{A_1}(g_1)-E_{A_1}(s)\rVert_2
}{
\lVert E_{A_1}(g_2)-E_{A_1}(s)\rVert_2
}
=
R,
\end{equation}
whereas under the second,
\begin{equation}
\frac{
\lVert E_{A_2}(g_1)-E_{A_2}(s)\rVert_2
}{
\lVert E_{A_2}(g_2)-E_{A_2}(s)\rVert_2
}
=
\frac{1}{R}.
\end{equation}
The two exact world models therefore induce opposite Euclidean distance
orderings. Since $R$ is arbitrary, the separation between the distances can
be made arbitrarily large.
\end{proof}

\subsection{Temporal reachability as a directed distance}
\label{app:theory-quasimetric}

Consider a deterministic controlled system with state space $\mathcal S$.
Define
\begin{equation}
d^\star(s,g)
:=
\inf\left\{
T\in\mathbb N_0:
\text{some length-$T$ action sequence takes $s$ to $g$}
\right\},
\end{equation}
with $d^\star(s,g)=+\infty$ when $g$ is unreachable from $s$.

\begin{proposition}[Temporal reachability is an extended directed quasimetric]
\label{prop:quasimetric}
For all states $s,y,g$,
\begin{equation}
d^\star(s,g)\geq 0,
\qquad
d^\star(s,g)=0 \iff s=g,
\qquad
d^\star(s,g)
\leq
d^\star(s,y)+d^\star(y,g).
\end{equation}
However, $d^\star(s,g)$ need not equal $d^\star(g,s)$. Consequently, when
temporal reachability is asymmetric, no symmetric distance such as a latent norm $\lVert E(s)-E(g)\rVert_2$ can represent it exactly on all ordered state
pairs.
\end{proposition}

\begin{proof}
The empty action sequence takes each state to itself, so
$d^\star(s,s)=0$. Conversely, a length-zero sequence cannot change the state,
so $d^\star(s,g)=0$ implies $s=g$. Nonnegativity follows because action
sequence lengths belong to $\mathbb N_0$.

The triangle inequality is immediate if either $d^\star(s,y)$ or
$d^\star(y,g)$ is infinite. Otherwise, concatenate a shortest sequence from
$s$ to $y$ with a shortest sequence from $y$ to $g$. The resulting sequence
takes $s$ to $g$ and has length
$d^\star(s,y)+d^\star(y,g)$.

Finally, consider two states for which an action takes $s$ to $g$, but no
action sequence returns from $g$ to $s$. Then
$d^\star(s,g)=1$ while $d^\star(g,s)=+\infty$. Because every symmetric
distance assigns the same value to $(s,g)$ and $(g,s)$, it cannot represent
$d^\star$ exactly in this case.
\end{proof}

\subsection{Formal Statement and Proof of
Theorem~\ref{thm:fixed-config-superiority}}
\label{app:heterogeneity}

We first make the ambient spaces and regularity conditions explicit. Let
\begin{equation}
\mathcal Z
\subseteq
\mathbb R^{d_z},
\qquad
\mathcal X
\subseteq
\mathcal Z\times\mathcal Z
\subseteq
\mathbb R^{2d_z},
\qquad
\Omega_{\mathrm{pl}}
\subseteq
\mathbb R^{d_\omega},
\qquad
\mathcal Y
\subseteq
\mathbb R^{d_y}.
\end{equation}
Here, \(\mathcal X\) is the task space, with \(x=(z_t,z_g)\in\mathcal X\), \(\Omega_{\mathrm{pl}}\) is the planner-state space, whose coordinates may include the current plan and bounded optimizer memory, and \(\mathcal Y\) is the output space of the planner interface. We assume that \(\mathcal X\) and \(\Omega_{\mathrm{pl}}\) are nonempty and compact, and equip all finite-dimensional spaces and product spaces with their Euclidean norms.

Let \(\mu\) be a probability distribution supported on \(\mathcal X\), let
\(\omega_0\in\Omega_{\mathrm{pl}}\) be the common planner initialization, and
let \(K<\infty\) be the number of refinement rounds. Assume that the abstract matched interface \(\mathcal I:\mathcal X\times\Omega_{\mathrm{pl}}\to\mathcal Y\) and \(\ell:\mathcal X\times\Omega_{\mathrm{pl}}\to\mathbb R\) are continuous.

Define the set of feasible planner inputs by
\begin{equation}
\mathcal W
=
\left\{
\left(
\omega,
\mathcal I(x,\omega)
\right):
x\in\mathcal X,\;
\omega\in\Omega_{\mathrm{pl}}
\right\}
\subseteq
\mathbb R^{d_\omega+d_y}.
\label{eq:feasible-planner-inputs}
\end{equation}
Because \(\mathcal X\times\Omega_{\mathrm{pl}}\) is compact and
\((x,\omega)\mapsto(\omega,\mathcal I(x,\omega))\) is continuous,
\(\mathcal W\) is compact.

Let
\begin{equation}
\mathfrak F_{\mathcal I}
=
\left\{
F:\mathcal W\to\Omega_{\mathrm{pl}}
\;:\;
F \text{ is continuous}
\right\}
\end{equation}
be the class of continuous feasible search rules. For any $F\in\mathfrak F_{\mathcal I}$, set
$
\omega_{F,0}(x)
=
\omega_0,
$
and, for \(k=0,\ldots,K-1\),
\begin{equation}
\omega_{F,k+1}(x)
=
F\!\left(
\omega_{F,k}(x),
\mathcal I\bigl(x,\omega_{F,k}(x)\bigr)
\right).
\label{eq:appendix-planner-recursion}
\end{equation}
Since every rule maps \(\mathcal W\) into
\(\Omega_{\mathrm{pl}}\), all planner states remain feasible.

The expected loss of \(F\) is
\begin{equation}
\mathcal L(F)
=
\mathbb E_{x\sim\mu}
\left[
\ell\bigl(x,\omega_{F,K}(x)\bigr)
\right].
\label{eq:appendix-planner-loss}
\end{equation}
Assumption~\ref{as:planner-universality} states that the abstract learned-planner class \(\{\mathcal F_\theta:\theta\in\Theta\}\subseteq\mathfrak F_{\mathcal I}\) receives the full feasible input \(w\in\mathcal W\) and that, for every \(F\in\mathfrak F_{\mathcal I}\) and every \(\varepsilon>0\), there exists \(\theta\in\Theta\) satisfying
\begin{equation}
\sup_{w\in\mathcal W}
\left\|
\mathcal F_\theta(w)-F(w)
\right\|
<
\varepsilon.
\label{eq:appendix-universal-approximation}
\end{equation}

Let \(\mathcal X_1,\ldots,\mathcal X_r\) be a measurable partition of
\(\mathcal X\), up to sets of \(\mu\)-measure zero, with
\begin{equation}
p_i
=
\mu(\mathcal X_i)
>
0.
\end{equation}
For every \(F\in\mathfrak F_{\mathcal I}\), define
\begin{equation}
\mathcal L_i(F)
=
\mathbb E
\left[
\ell\bigl(x,\omega_{F,K}(x)\bigr)
\,\middle|\,
x\in\mathcal X_i
\right].
\end{equation}

Let
\begin{equation}
\mathcal B
\subseteq
\mathbb R^{d_B}
\end{equation}
be a nonempty compact family of fixed search configurations. Each
\(B\in\mathcal B\) induces a rule
\(F_B\in\mathfrak F_{\mathcal I}\), with the same configuration \(B\) used
on every task. Define
\begin{equation}
\mathcal L_i(B)
=
\mathcal L_i(F_B),
\qquad
\mathcal L(B)
=
\sum_{i=1}^r p_i\mathcal L_i(B),
\end{equation}
and assume that \(B\mapsto\mathcal L_i(B)\) is continuous for every region
\(i\). Consequently, the regional minimum
\begin{equation}
b_i
:=
\min_{B\in\mathcal B}
\mathcal L_i(B)
\label{eq:regional-minimum}
\end{equation}
exists for every \(i\).

\begin{assumption}[Regional incompatibility]
\label{as:regional-incompatibility}
No fixed configuration minimizes every regional loss:
\begin{equation}
\bigcap_{i=1}^r
\operatorname*{argmin}_{B\in\mathcal B}
\mathcal L_i(B)
=
\varnothing.
\label{eq:incompatible-regional-optima}
\end{equation}
\end{assumption}

\begin{assumption}[Interface composability]
\label{as:interface-composability}
There exists a continuous feasible rule
\(F^\star\in\mathfrak F_{\mathcal I}\) that attains the best fixed-configuration
loss in every region:
\begin{equation}
\mathcal L_i(F^\star)
=
b_i
=
\min_{B\in\mathcal B}
\mathcal L_i(B),
\qquad
i=1,\ldots,r.
\label{eq:interface-composability}
\end{equation}
\end{assumption}

\begin{theorem}[Strict advantage under task heterogeneity;
restatement of Theorem~\ref{thm:fixed-config-superiority}]
Under Assumption~\ref{as:planner-universality},
Assumption~\ref{as:regional-incompatibility}, and
Assumption~\ref{as:interface-composability},
\begin{equation}
\inf_{\theta\in\Theta}
\mathcal L(\mathcal F_\theta)
\leq
\sum_{i=1}^r
p_i
\min_{B\in\mathcal B}
\mathcal L_i(B)
<
\min_{B\in\mathcal B}
\mathcal L(B).
\end{equation}
\end{theorem}

\begin{proof}
By Assumption~\ref{as:interface-composability}, there exists
\(F^\star\in\mathfrak F_{\mathcal I}\) such that
\(\mathcal L_i(F^\star)=b_i\) for every \(i\). Since the regions partition
the task distribution,
\begin{equation}
\mathcal L(F^\star)
=
\sum_{i=1}^r p_i\mathcal L_i(F^\star)
=
\sum_{i=1}^r p_i b_i.
\label{eq:regional-composite-loss}
\end{equation}

We next show that the neural-planner class can approach this loss. For each
integer \(m\geq1\), apply
Assumption~\ref{as:planner-universality} with
\(\varepsilon=1/m\). This gives a parameter
\(\theta_m\in\Theta\) satisfying
\begin{equation}
\sup_{w\in\mathcal W}
\left\|
\mathcal F_{\theta_m}(w)-F^\star(w)
\right\|
<
\frac{1}{m}.
\label{eq:uniform-planner-approximation}
\end{equation}
For brevity, write
\begin{equation}
\omega_{m,k}(x)
=
\omega_{\mathcal F_{\theta_m},k}(x),
\qquad
\omega^\star_k(x)
=
\omega_{F^\star,k}(x).
\end{equation}
We prove by induction that, for every fixed \(k\leq K\),
\begin{equation}
\sup_{x\in\mathcal X}
\left\|
\omega_{m,k}(x)-\omega^\star_k(x)
\right\|
\longrightarrow
0
\qquad
\text{as }m\to\infty.
\label{eq:planner-state-convergence}
\end{equation}

The claim holds for \(k=0\), because all planners share the initialization
\(\omega_0\). Suppose that it holds at round \(k\). Define
\begin{align}
w_{m,k}(x)
&=
\left(
\omega_{m,k}(x),
\mathcal I\bigl(x,\omega_{m,k}(x)\bigr)
\right),
\\
w^\star_k(x)
&=
\left(
\omega^\star_k(x),
\mathcal I\bigl(x,\omega^\star_k(x)\bigr)
\right).
\end{align}
Continuity of \(\mathcal I\) on the compact set
\(\mathcal X\times\Omega_{\mathrm{pl}}\) implies uniform continuity.
Therefore, the induction hypothesis gives
\begin{equation}
\sup_{x\in\mathcal X}
\left\|
w_{m,k}(x)-w^\star_k(x)
\right\|
\longrightarrow
0.
\label{eq:planner-input-convergence}
\end{equation}

Using the planner recursion and adding and subtracting
\(F^\star(w_{m,k}(x))\), we obtain
\begin{align}
&
\sup_{x\in\mathcal X}
\left\|
\omega_{m,k+1}(x)-\omega^\star_{k+1}(x)
\right\|
\\
&\quad\leq
\sup_{x\in\mathcal X}
\left\|
\mathcal F_{\theta_m}\bigl(w_{m,k}(x)\bigr)
-
F^\star\bigl(w_{m,k}(x)\bigr)
\right\|
\nonumber\\
&\qquad\quad+
\sup_{x\in\mathcal X}
\left\|
F^\star\bigl(w_{m,k}(x)\bigr)
-
F^\star\bigl(w^\star_k(x)\bigr)
\right\|.
\label{eq:planner-induction-bound}
\end{align}
The first term is at most \(1/m\) by
Eq.~\ref{eq:uniform-planner-approximation}. The second converges to zero
because \(F^\star\) is uniformly continuous on the compact set
\(\mathcal W\) and Eq.~\ref{eq:planner-input-convergence} holds. This proves
Eq.~\ref{eq:planner-state-convergence} for every finite
\(k\leq K\).

Because \(\ell\) is continuous on the compact set
\(\mathcal X\times\Omega_{\mathrm{pl}}\), it is uniformly continuous.
Applying Eq.~\ref{eq:planner-state-convergence} at \(k=K\) therefore yields
\begin{equation}
\sup_{x\in\mathcal X}
\left|
\ell\bigl(x,\omega_{m,K}(x)\bigr)
-
\ell\bigl(x,\omega^\star_K(x)\bigr)
\right|
\longrightarrow
0.
\end{equation}
Consequently,
\begin{equation}
\mathcal L(\mathcal F_{\theta_m})
\longrightarrow
\mathcal L(F^\star).
\end{equation}
Together with Eq.~\ref{eq:regional-composite-loss}, this gives
\begin{equation}
\inf_{\theta\in\Theta}
\mathcal L(\mathcal F_\theta)
\leq
\sum_{i=1}^r p_i b_i.
\label{eq:regional-learned-bound}
\end{equation}

It remains to show that every single fixed configuration has strictly larger
expected loss. Define its excess over the regional minima by
\begin{equation}
\Delta(B)
=
\sum_{i=1}^r
p_i
\bigl(
\mathcal L_i(B)-b_i
\bigr).
\end{equation}
Every term in this sum is nonnegative. By
Assumption~\ref{as:regional-incompatibility}, each
\(B\in\mathcal B\) is strictly suboptimal in at least one region. Since every
\(p_i>0\),
\begin{equation}
\Delta(B)>0
\qquad
\text{for every }B\in\mathcal B.
\end{equation}
The function \(\Delta\) is continuous because it is a finite weighted sum of
the continuous functions \(\mathcal L_i\). Since \(\mathcal B\) is compact,
\(\Delta\) attains its minimum. Its pointwise strict positivity implies
\begin{equation}
\eta
:=
\min_{B\in\mathcal B}
\Delta(B)
>
0.
\end{equation}
Hence
\begin{align}
\min_{B\in\mathcal B}
\mathcal L(B)
&=
\min_{B\in\mathcal B}
\left[
\sum_{i=1}^r p_i b_i+\Delta(B)
\right]
\\
&=
\sum_{i=1}^r p_i b_i+\eta
\\
&>
\sum_{i=1}^r p_i b_i.
\label{eq:strict-fixed-gap}
\end{align}
Combining Eq.~\ref{eq:regional-learned-bound} with
Eq.~\ref{eq:strict-fixed-gap} proves
\begin{equation}
\inf_{\theta\in\Theta}
\mathcal L(\mathcal F_\theta)
\leq
\sum_{i=1}^r
p_i
\min_{B\in\mathcal B}
\mathcal L_i(B)
<
\min_{B\in\mathcal B}
\mathcal L(B).
\end{equation}
\end{proof}

\newpage
\section{Method Details}
\label{app:method}

\subsection{Value Learning}
\label{sec:value-learning}
For each environment and world model, we train a separate goal-conditioned
cost-to-go function~\citep{kaelbling1993learning,schaul2015universal}
\begin{equation}
V_\psi(z,z_g)\colon
\mathcal Z\times\mathcal Z\rightarrow\mathbb R_{\geq0}.
\end{equation}
Lower values represent shorter predicted temporal
distance~\citep{hartikainen2020dynamical} to the goal, as we assume a cost of $1$ per step. The
world-model encoder is frozen, and the value function is trained entirely from
cached offline latents.
The value is represented by a metric residual
network~\citep{liu2023metric,wang2022improved},
\begin{equation}
V_\psi(z,z_g)
=
\lVert u_\psi(z)-u_\psi(z_g)\rVert_2
+
\max_j
\operatorname{ReLU}
\left(
v_{\psi,j}(z_g)-v_{\psi,j}(z)
\right).
\label{eq:value-quasimetric}
\end{equation}
Here $u_\psi$ is the first half of the latent vector the critic head computes and $v_\psi$ is the second half. The first term is symmetric, while the second permits directed temporal
distance~\citep{wang2023optimal}.
For each update, we sample an anchor $z_t$, an $n$-step successor
$z_{t+n_{\mathrm{eff}}}$, and a hindsight
goal~\citep{andrychowicz2017hindsight} $z_g$, where $n_{\mathrm{eff}}
=
\min\{n,T_{\mathrm{episode}}-t\}.$
In-episode goals are sampled from future states with temporal offsets balanced
across the available episode horizon. Cross-episode goals are additionally
sampled to train long-range state pairs.
If an in-episode goal lies within the backup window, its exact temporal
distance $\delta$ is used. Otherwise, the target is bootstrapped with an
$n$-step backup~\citep{sutton2018reinforcement}:
\begin{equation}
y_t
=
\begin{cases}
\delta,
&
\delta\leq n_{\mathrm{eff}},
\\[3pt]
c_\gamma(n_{\mathrm{eff}})
+
\gamma^{n_{\mathrm{eff}}}
\bar V_{\bar\psi}(z_{t+n_{\mathrm{eff}}},z_g),
&
\text{otherwise},
\end{cases}
\label{eq:value-target}
\end{equation}
with
\begin{equation}
c_\gamma(n)
=
\sum_{i=0}^{n-1}\gamma^i
=
\begin{cases}
n, & \gamma=1,\\[2pt]
\dfrac{1-\gamma^n}{1-\gamma}, & \gamma<1.
\end{cases}
\end{equation}
The target parameters~\citep{mnih2015human} are updated by Polyak
averaging~\citep{lillicrap2016continuous},
\begin{equation}
\bar\psi
\leftarrow
(1-\eta)\bar\psi+\eta\psi.
\end{equation}
Following implicit Q-learning~\citep{kostrikov2022offline,park2023hiql}, the
value function is trained by asymmetric expectile
regression~\citep{newey1987asymmetric}, replacing the squared penalty with a
Huber penalty~\citep{huber1964robust} for robustness:
\begin{equation}
\mathcal L_V(\psi)
=
\mathbb E_{(z_t, z_g) \sim \mathcal D} \Bigl[
\bigl|\tau-\textbf{1}\!\left[V_\psi(z_t,z_g)-y_t>0\right]\bigr|
\,\ell_{\mathrm{Huber}}\!\left(V_\psi(z_t,z_g)-y_t\right)
\Bigr].
\label{eq:value-loss}
\end{equation}
Since $V_\psi$ is a cost-to-go rather than a return, we use $\tau<0.5$: the
weight on overestimation exceeds the weight on underestimation, so $V_\psi$
regresses toward a lower expectile of the target distribution, approximating
the shortest temporal distance realizable in the data rather than the
behavior-policy average.

\subsection{RP1 Training}
\label{sec:RP1-training}

RP1 is trained entirely offline while the world-model encoder and dynamics
predictor remain frozen. For each world model, the planner is trained from a
stride-five latent cache aligned with five-step action blocks. RP1 consists of three fully-connected layers with ReLU activations and hidden
width $512$, mapping
$\mathbb{R}^{2N|a|+1}\rightarrow\mathbb{R}^{512}\rightarrow\mathbb{R}^{512}
\rightarrow\mathbb{R}^{N|a|}$, where the input concatenates the flattened plan
$\mathbf a_k\in\mathbb{R}^{N|a|}$, its value gradient
$\mathbf g_k\in\mathbb{R}^{N|a|}$, and the scalar terminal value $v_k$, and the
output is the residual plan update. For example in OGBench Cube, with a planning horizon of $N=5$ action
blocks and $|a|=25$ (five primitive steps of the five-dimensional arm actions), the refiner is $251\rightarrow512\rightarrow512\rightarrow125$, i.e.\
$0.46$M parameters, applied with tied weights at all $K=8$ refinement steps.

The RP1 actor is a weight-tied residual plan refiner. At refinement step $k$,
it receives the current plan, its terminal value, and the value gradient with
respect to the plan:
\begin{align}
\hat z_N^{(k)}
&=
H_\phi(\mathbf a_k,z_0),
\\
v_k
&=
V_{\bar\psi}(\hat z_N^{(k)},z_g),
\\
\mathbf g_k
&=
\nabla_{\mathbf a_k}
V_{\bar\psi}(\hat z_N^{(k)},z_g).
\end{align}
The plan is updated by
\begin{equation}
\mathbf a_{k+1}
=
\operatorname{clip}_{[-a_{\max},a_{\max}]}
\left[
\mathbf a_k+
f_\theta(\mathbf a_k,v_k,\mathbf g_k)
\right].
\label{eq:RP1-training-update}
\end{equation}
The actor receives no raw current-state or goal latent. Goal information reaches
it only through $v_k$ and $\mathbf g_k$.

The actor is trained by differentiating the terminal value through the frozen
world-model rollout. The value and gradient supplied as refiner inputs are
detached, while the training loss remains differentiable through the refined
action sequence and its resulting rollout. No environment interaction is used
during this stage.

Let $v_k$ be the terminal value after refinement step $k$. The planner objective
is
\begin{equation}
 J_{\mathrm{RP1}}(\theta)
=
\mathbb{E}_{\hat z^{(K)}}\Bigl [
v_K
+
\lambda_{\mathrm{mean}}
\frac{1}{K}
\sum_{k=1}^{K}v_k \Bigr ].
\label{eq:RP1-training-objective}
\end{equation}
The initial value from Section~\ref{sec:value-learning} initializes the RP1
critic. When critic co-training is enabled, it continues to receive the same
cached-data TD updates while a Polyak-averaged copy supplies $v_k$ and
$\mathbf g_k$.

All reported experiments execute each plan open loop before replanning. For closed-loop control, one could let $\hat \tau$ be the first imagined goal-arrival step, measured by $v_k \leq \epsilon$, or $N$ if the goal is not reached, and optimize the telescoped per-step costs
$
\sum_{i=0}^{\hat\tau-1}\gamma^i
(
1+\gamma V(\hat z_{t+i+1}^{(K)},z_g)
-V(\hat z_{t+i}^{(K)},z_g)
)$,
which yields an arrival-aware loss that favors reaching the goal earlier. We do not evaluate this variant.

\subsection{Dyna Loop}
\label{sec:dyna}

The values $v_k$
in the planner-loss are read off latents $H_\phi(\mathbf a_k, \hat z_0)$
that the world models $h_\phi$ produced. Should $h_\phi$ be wrong, or not have coverage for the dataset $\mathcal D$, the planner can exploit inaccuracies, as is well reported in literature \cite{smith2006optimizer, talvitie2017self,
jafferjee2020hallucinating}. 

Much of this can be fixed by finetuning the world model on actual roll-out data, as originally proposed in the Dyna loop\cite{sutton1991dyna}. For this we deploy $\theta_r$ in the real
environment, collect the (failure) trajectories it produces, mix them into the training data, and finetune the world model on the mixture. Then we retrain the planner and repeat.

\newpage
%
\section{Empirical Results}
\label{app:results}


\subsection{General Setup}
\label{app:results-general}

\paragraph{Data and evaluation protocol.}
All world-model encoders and dynamics predictors are frozen throughout planner training. Each domain provides $10{,}000$ episodes: value functions and planners train on episodes $0$--$7{,}999$, while all evaluations draw start/goal states from the held-out episodes $8{,}000$--$9{,}999$. All model-selection decisions, including checkpoint selection and the stopping points in Table~\ref{tab:stopping-points}, are made exclusively on selection draws $\{48,49,50,51\}$, which are never used for reported results. Selection evaluates $36$ checkpoint pairs for six seeds on four draws of $50$ episodes, i.e.\ $43{,}200$ simulator episodes per experimental cell. These episodes are used only to choose checkpoints; no gradient update uses environment data. Final reported numbers use the disjoint evaluation draws $\{42,43,44\}$, with $50$ episodes per draw, and are the median over six planner training seeds $\{0,\dots,5\}$ of each seed's mean across those three evaluation draws.

\paragraph{Planning protocol.}
All planners use $5$-step action chunks ($25$ primitive
steps) and optimize them open loop, replanning every $5$ chunks (receding horizon $5$). The goal state is $h$ primitive steps ahead and the episode budget has the length of $2h$ steps; TwoRoom and Cube
evaluate $h\in\{25,100\}$ ($h25$, $h100$), Reacher $h{=}25$. Simulator evaluations run
under EGL with a pinned render device and at most one evaluation per GPU; we verified
that this parallelism is bit-identical to fully serialized evaluation before adopting it.

\paragraph{Objectives.}
Every planner scores the predicted terminal state with one of the two objectives of
Sec.~\ref{sec:baseline-hyperparameters}: the latent-distance objective
$C_{\mathrm{latent}}(\hat z_N,z_g)=\lVert\hat z_N-z_g\rVert_2^2$, or the value objective
$C_{\mathrm{value}}(\hat z_N,z_g)=V_\psi(\hat z_N,z_g)$, the goal-conditioned
temporal-distance critic of Sec.~\ref{sec:value-learning} (MRN quasimetric-style head)
trained on the frozen cached latents of each base. All encoders and dynamics predictors remain frozen during critic and planner training, so differences in performance arise from how imagined trajectories are scored and improved rather than from changes to the world models.

\paragraph{Shared recipe across domains.}
Every RP1 result uses the same actor architecture, critic architecture, refiner architecture, optimization settings, and planner hyperparameters shown in Table~\ref{tab:rp1-hypers-uniJ} across all three environments and both world-model backbones. We make only two environment-independent adjustments dictated by the prediction horizon and observability. First, the critic discount is set by the horizon being evaluated, with $\gamma{=}0.98$ for $h25$ and $\gamma{=}0.99$ for $h100$, so that its effective horizon $1/(1-\gamma)$ covers the task horizon. Second, Reacher uses a two-frame latent cost representation for both the critic and latent-distance objective because joint velocity is otherwise unobservable from a single frame.

\paragraph{Stopping point.}
We train RP1 with early stopping. The offline critic is snapshotted every $3{,}000$ steps up to its $18{,}000$-step cap, and the planner every $2{,}000$ steps up to its $12{,}000$-step cap. Each of the $6\times6$ planner--critic checkpoint pairs is evaluated for all six training seeds on the four selection draws $\{48,49,50,51\}$. For each experimental cell, we select a single checkpoint pair by its mean score across seeds and selection draws. Ties are broken toward the smaller training budget and then the earlier step. This selected pair is then fixed, and all six seeds are evaluated on the disjoint reporting draws $\{42,43,44\}$. No per-seed checkpoint selection is performed, and the same selection rule is applied to the learned baselines.

\begin{table}[H]
  \centering
  \captionsetup{font=small,skip=3pt}
  \footnotesize
  \setlength{\tabcolsep}{10pt}
  \renewcommand{\arraystretch}{1.15}
  \begin{tabular}{@{}llcccc@{}}
  \toprule
  & & \multicolumn{2}{c}{$25$ steps} & \multicolumn{2}{c}{$100$ steps}\\
  \cmidrule(lr){3-4}\cmidrule(lr){5-6}
  domain & base & critic & planner & critic & planner\\
  \midrule
  TwoRoom & LeWM & $15$ & $2$  & $6$  & $8$\\
          & PLDM & $3$  & $2$  & $3$  & $10$\\
  \addlinespace[2pt]
  Reacher & LeWM & $18$ & $2$  & --- & ---\\
          & PLDM & $18$ & $2$  & --- & ---\\
  \addlinespace[2pt]
  Cube    & LeWM & $3$  & $12$ & $9$  & $2$\\
          & PLDM & $15$ & $8$  & $12$ & $2$\\
  \bottomrule
  \end{tabular}
  \vspace{3pt}
  \caption{\textbf{Selected stopping points}. The offline critic is
  snapshotted every $3$k to its $18$k cap and the planner every $2$k to its $12$k cap;
  one of the $6\times6$ pairs is fixed per cell as described above. Reacher is evaluated
  at $h25$ only. A planner entry of $12$ is the step cap, i.e.\ the final iterate.}
  \label{tab:stopping-points}
\end{table}

\begin{table}[H]
\centering
\footnotesize
\renewcommand{\arraystretch}{1.18}
\setlength{\tabcolsep}{5pt}
\begin{tabular}{@{}ll@{\hspace{2.2em}}ll@{}}
\toprule
\textbf{Hyperparameter} & \textbf{Value} & \textbf{Hyperparameter} & \textbf{Value}\\
\midrule
\multicolumn{2}{@{}l}{\textbf{Plan refiner}} & \multicolumn{2}{@{}l}{\textbf{Co-trained critic}}\\
refinement iterations $K$      & $8$                  & $\gamma$ ($h25$ / $h100$)      & $0.98$ / $0.99$\\
plan horizon $H$ (chunks)      & $5$                  & $n$-step                       & $1$\\
replan every                   & $5$ chunks           & expectile (annealed)           & $0.1{\to}0.03$\\
clip range $a_{\max}$          & $2.5$                & critic LR (annealed)           & $10^{-3}{\to}10^{-4}$\\
refiner MLP (width $\times$ hidden layers) & $512\times2$    & live fraction of actor steps   & $0.8$\\
                               &                      & EMA $\tau$ / TD batch          & $0.005$ / $1{,}024$\\
\addlinespace[3pt]
\multicolumn{2}{@{}l}{\textbf{Actor}} & head & MRN, $256/128$, depth $2$\\
mean-weight $\lambda_{\mathrm{mean}}$ & $0.3$          & \multicolumn{2}{@{}l}{}\\
actor LR (cosine to $1/10$)    & $3{\cdot}10^{-4}$    & \multicolumn{2}{@{}l}{\textbf{Offline teacher}}\\
batch size                     & $256$                & step cap / snapshot            & $18{,}000$ / $3{,}000$\\
step cap / snapshot            & $12{,}000$ / $2{,}000$ & expectile                    & $0.03$\\
max-delta (chunks)             & $20$                 & teachers per actor seed        & $1$\\
cross-episode goal probability & $0.3$                & \multicolumn{2}{@{}l}{}\\
anti-constancy weight          & $0.5$                & \multicolumn{2}{@{}l}{}\\

\bottomrule
\end{tabular}
\caption{\textbf{The RP1 configuration.} }
\label{tab:rp1-hypers-uniJ}
\end{table}

\newpage

\subsection{TwoRoom}
\label{app:results-tworoom}

\paragraph{Setup.}
In TwoRoom, success is judged by whether the final distance
to the goal is within $16$ pixels. Figure~\ref{fig:tworoom-fields} probes
the learned critic, comparing latent distance to the critic's value
landscape on sampled tasks; Fig.~\ref{fig:tworoom-trajectories} traces plan
refinement against the hand-designed planners.

\begin{figure}[H]
  \centering
  \includegraphics[width=0.8\linewidth]{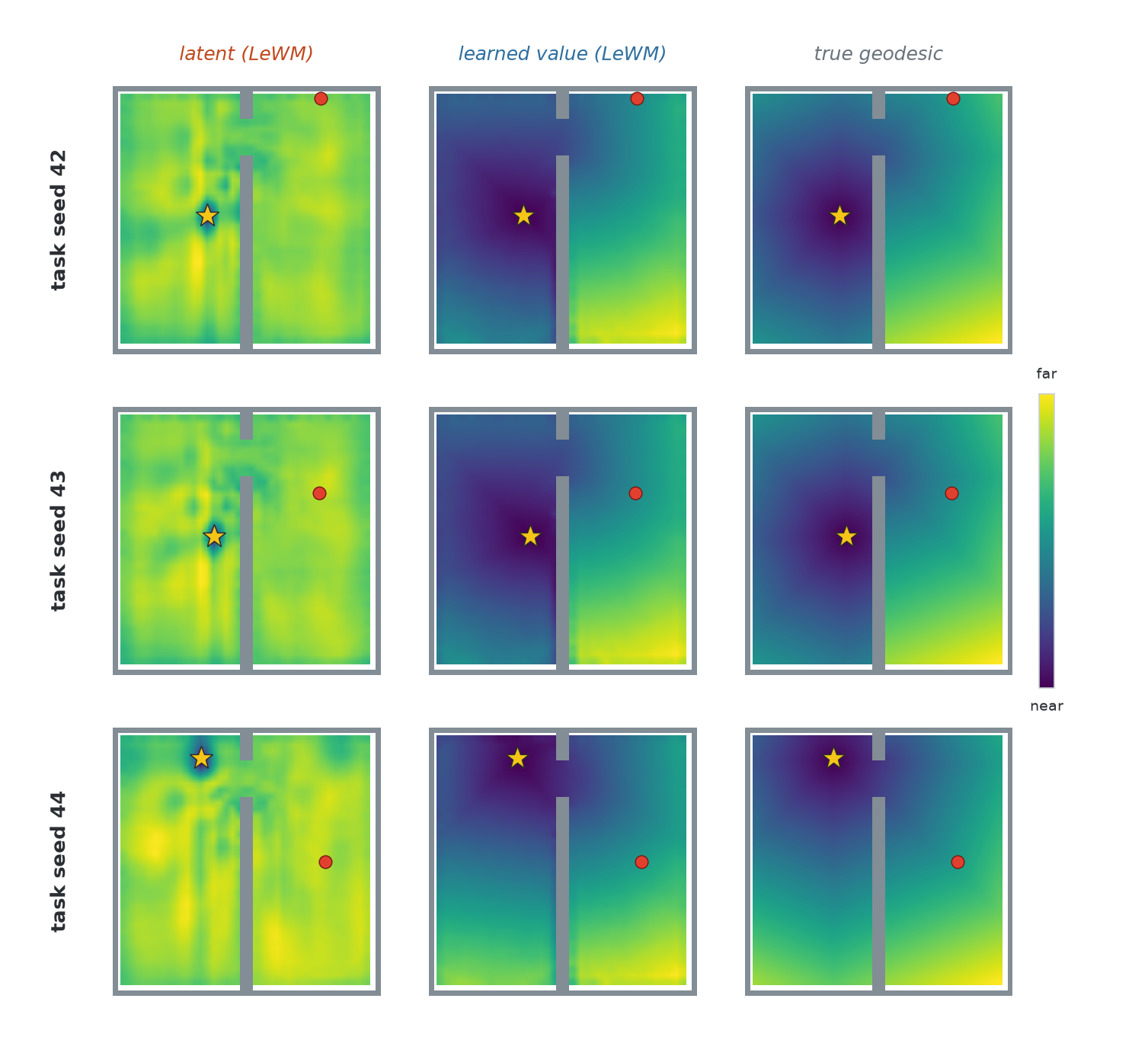}
  \caption{\textbf{Latent-Distance vs. learned Cost-to-go.} 3 randomly sampled tasks from seeds (42,43,44) and their corresponding cost landscapes.}
  \label{fig:tworoom-fields}
\end{figure}
\begin{figure}[p]
  \centering
  \includegraphics[width=0.85\linewidth]{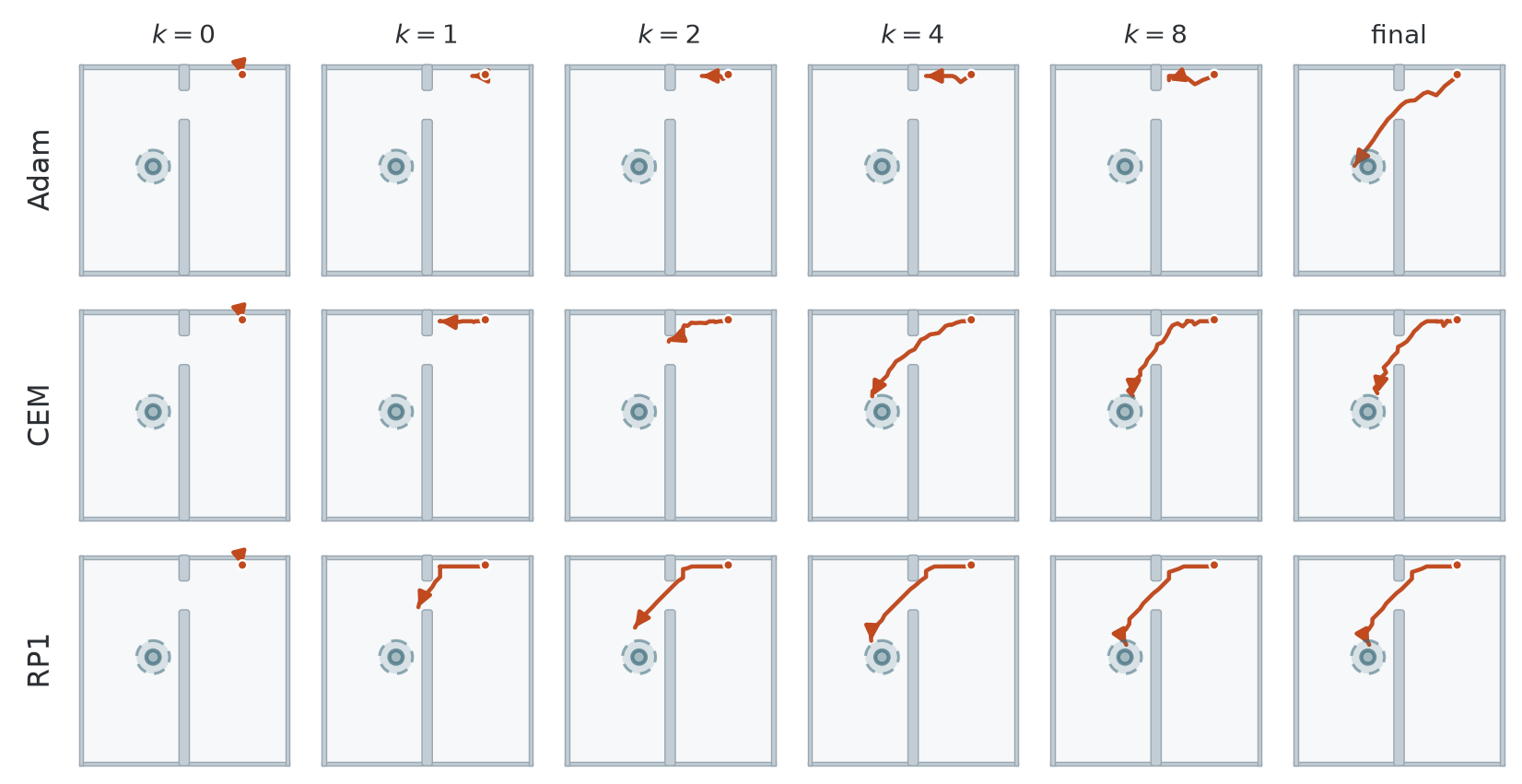}\\[8pt]
\includegraphics[width=0.85\linewidth]{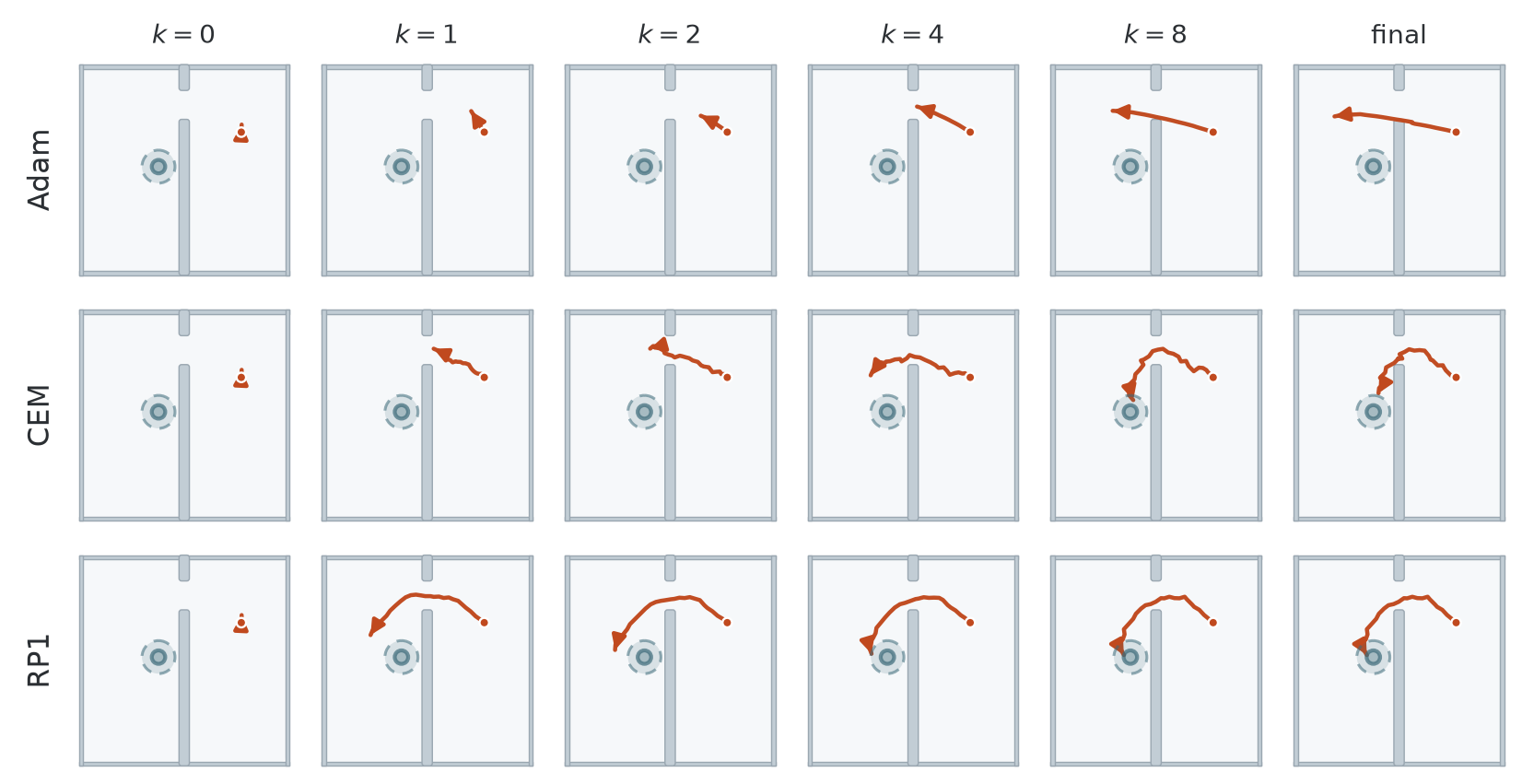}\\[8pt]
\includegraphics[width=0.85\linewidth]{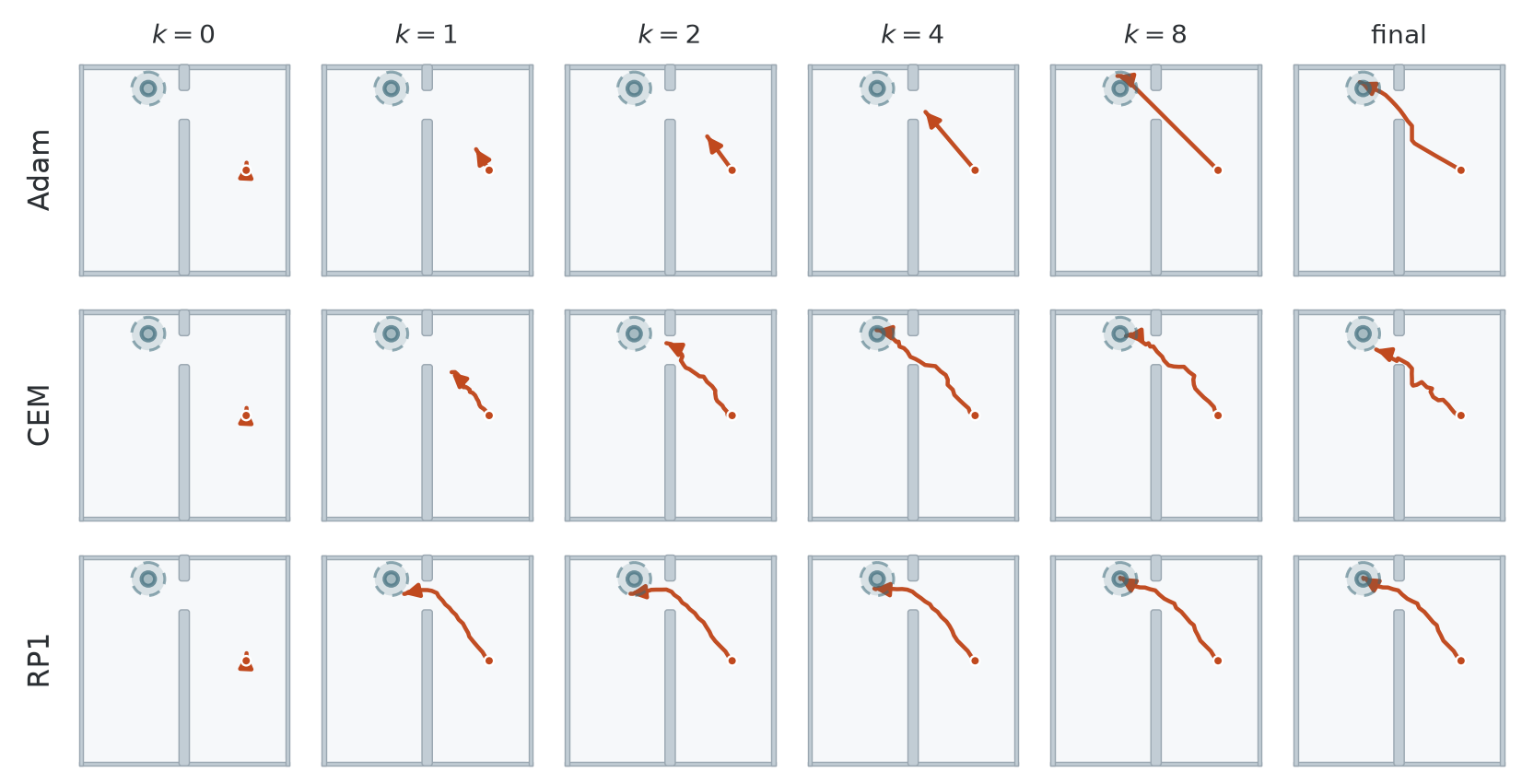}
  \caption{\textbf{Plan refinement in TwoRoom.} Each panel shows the
planner's best-scoring candidate plan at refinement iteration $k$ on the
same task, with all planners scoring plans using the learned value critic;
``final'' is $k{=}30$ for Adam and CEM and $k{=}8$ for RP1. Iterations
differ greatly in cost: one CEM iteration evaluates $300$ sampled plans
($9{,}000$ rollouts in total), one Adam iteration takes a gradient step on
$100$ plans in parallel ($3{,}000$ rollouts), whereas one RP1 iteration is
a single forward pass of the learned refiner costing one rollout ($9$ in
total, including the initial evaluation).}
  \label{fig:tworoom-trajectories}
\end{figure}
\newpage

\subsection{Reacher}
\label{app:results-reacher}

\paragraph{Setup.}
In Reacher success is first-hit: all joints
within $\tau$ radians of the goal configuration, scored in a separate
simulator pass per $\tau$ with termination on success.

\paragraph{Cost windows.}
On Reacher, both objectives read the predicted terminal state through a
latent window of $w$ terminal frames; we use $w{=}1$ for all reported results except Reacher, where we use $w = 2$, as joint velocity is not observable from a single frame, and a window of $w>1$ lets the cost read it. The Reacher value critic
additionally uses window lag $5$ and standardized latents, and the RP1
co-trained critic is initialized from the offline value trained at the
same cost window.

\subsection{OGBench Cube}
\label{app:results-ogbench}

\paragraph{Setup.}
Every cell is $50$ episodes per evaluation seed. Success follows the benchmark's
cube-placement criterion; the no-op floors ($56.0$ at $h25$, $45.3$ at $h100$) are
measured by executing zero actions under the identical protocol.

\paragraph{Dyna iterations.}
In each Dyna round, we collect on-policy episodes with the deployed planners of the previous
round (all six seeds) on the training-split task. We
collect at \emph{both} evaluation horizons, $1{,}800$ episodes at $h25$ and $1{,}800$ at
$h100$, and mix them $50{:}50$ with the original expert data. The world
model is then fine-tuned for one epoch at LR $10^{-5}$ from the previous round's model,
with a latent anchor $\lVert \mathrm{enc}_{\mathrm{ft}}(x) - \mathrm{enc}_{\mathrm{prev}}(x)\rVert^2$
(weight $1.0$) added to the objective on every frame of the mixture. RP1 is retrained at the cell's locked
configuration, the teacher budget and actor step pair selected in
Sec.~\ref{app:results-general}. 

We run two rounds and report, per cell, the round with the better validation score:
round~$1$ for LeWM ($93.00$ against round~2's $92.75$) and round~$2$ for PLDM ($90.33$
against $83.25$). We report Dyna at $h25$ only.

\subsection{Planning Baselines}
\label{sec:baseline-hyperparameters}
Each conventional planner is evaluated with two terminal objectives. The
latent-distance objective scores the predicted terminal latent by
\begin{equation}
C_{\mathrm{latent}}(\hat z_N,z_g)
=
\lVert \hat z_N-z_g\rVert_2^2.
\end{equation}
The value objective uses the goal-conditioned value trained for the
corresponding environment and world model:
\begin{equation}
C_{\mathrm{value}}(\hat z_N,z_g)
=
V_\psi(\hat z_N,z_g).
\end{equation}
All baselines plan in the same normalized $5$-chunk action space as RP1 and
follow the identical receding-horizon protocol; they differ only in how the
action sequence is optimized. Per decision, CEM and MPPI evaluate $9{,}000$
forward rollouts; Adam evaluates $3{,}000$ forward rollouts and the
corresponding $3{,}000$ backward passes.

\paragraph{Cross Entropy Method (CEM).}
CEM samples complete action sequences from a factorized Gaussian, retains the
lowest-cost elite set, and refits the sampling distribution after every
iteration. We use $300$ samples per iteration for $30$ iterations with an
elite set of $30$ (top $10\%$); the initial distribution is zero-mean with
unit variance in the normalized action space.
\paragraph{Model-Predictive Path-Integral (MPPI).}
MPPI samples Gaussian perturbations around the current action sequence and
updates the sequence using exponentially weighted trajectory costs. We use
$300$ samples per iteration for $30$ iterations with temperature
$\lambda=0.5$.
\paragraph{Adam.}
Adam directly differentiates the terminal objective through the frozen
world-model rollout and optimizes a batch of action sequences. We optimize
$300$ sequences in parallel for $10$ steps with AdamW at learning rate
$0.1$ and execute the lowest-cost sequence. For TwoRoom we optimized $100$ sequences in parallel at $30$ steps.

\paragraph{Deep Model-Predictive Optimization (DMPO).}
DMPO learns a residual on the MPPI update: a network reads the sampling
distribution and the $N$ rollout costs and emits a gated mean
correction, a covariance update, and a learned warm-start shift \citep{sacks2024deep}.
We train it with the original objective, PPO, inside the frozen world model
against the same critic $V_\psi$ used by RP1 in each cell, for the original budget of
$1{,}000$ policy-improvement iterations, with six training seeds. Snapshots are taken every $200$
iterations, and one shared snapshot per cell is chosen by the mean score over seeds
on the selection draws, with ties broken toward the earlier snapshot. This is the same early-stopping rule used to select RP1's checkpoints. DMPO uses $256$ rollouts per decision ($256{\times}1$).

\paragraph{Learning-to-Optimize MPC (L2O-MPC).}
L2O-MPC learns to imitate the full sampling update of MPPI rather than a
residual: a network reads the mean, covariance, and $N$ costs and emits a gated
replacement mean \citep{sacks2022learning}. It is trained by DAgger imitation of a larger-budget MPPI expert, computed here through
the same frozen world model and critic. We follow the original DAgger structure:
$20$ rounds, $\beta_k = 0.8^k$, and an expert budget well above the learner's ($N{=}512$
against $M{=}64$). The original optimizer budget of roughly $10^6$ gradient steps per seed was not necessary, as we have found that training for $18{,}000$ steps already plateaued. We recorded snapshots every $3{,}000$ steps, and apply the same selection rule as for RP1. L2O-MPC uses $256$ rollouts per decision ($64{\times}4$), matching DMPO.

\end{document}